\documentclass[letterpaper]{article} 
\usepackage{main}  
\usepackage{times}  
\usepackage{helvet}  
\usepackage{courier}  
\usepackage[hyphens]{url}  
\usepackage{graphicx} 
\usepackage{natbib}  
\usepackage{caption} 
\usepackage{algorithm}
\usepackage{algorithmic}

\usepackage{newfloat}
\usepackage{listings}
\DeclareCaptionStyle{ruled}{labelfont=normalfont,labelsep=colon,strut=off} 
\floatstyle{ruled}
\newfloat{listing}{tb}{lst}{}
\floatname{listing}{Listing}

\usepackage{amsmath}
\usepackage{amsthm}
\usepackage{amssymb}
\newtheorem{definition}{Definition}
\newtheorem{theorem}{Theorem}
\newtheorem{example}{Example}

\newtheorem{proposition}{Proposition}

\usepackage{booktabs}     
\usepackage{makecell}     

\usepackage{tikz}
\usepackage{multirow} 
\usetikzlibrary{positioning}
\newenvironment{myproof}
  {\par\noindent\textbf{Proof Sketch.}\ }
  {\hfill$\square$\par}

  \usepackage{amsmath}
\usepackage{amsthm}
\usepackage{amssymb}

\usepackage{graphicx}   
\usepackage{subcaption} 

\usepackage{booktabs}     
\usepackage{makecell}     

\usepackage{tikz}
\usepackage{multirow} 
\usetikzlibrary{positioning}

\title{Diminution: On Reducing the Size of Grounding  ASP Programs}

\author{
  Huanyu Yang\textsuperscript{\rm 1},
  Fengming Zhu\textsuperscript{\rm 2},
  Yangfan Wu\textsuperscript{\rm 2},
  Jianmin Ji\textsuperscript{\rm 1}\thanks{Corresponding author.}
}
\affiliations{
  \textsuperscript{\rm 1}School of Computer Science and Technology, University of Science and Technology of China (USTC)\\
  Hefei, Anhui, China\\
  \textsuperscript{\rm 2}Department of Computer Science and Engineering, The Hong Kong University of Science and Technology (HKUST)\\
  Hong Kong SAR, China\\
  yanghuanyu@mail.ustc.edu.cn, jianmin@ustc.edu.cn,  fzhuae@connect.ust.hk, yangfan.wu@connect.ust.hk
}

\begin{document}
\nocopyright 
\maketitle

\begin{abstract}
Answer Set Programming (ASP) is often hindered by the \emph{grounding bottleneck}: large Herbrand universes generate ground programs so large that solving becomes difficult. Many methods employ ad-hoc heuristics to improve grounding performance, motivating the need for a more formal and generalizable strategy. We introduce the notion of \emph{diminution}, defined as a selected subset of the Herbrand universe used to generate a reduced ground program before solving.
We give a formal definition of diminution, analyze its key properties, and study the complexity of identifying it. 
We use a specific encoding that enables off‑the‑shelf ASP solver to evaluate candidate subsets. Our approach integrates seamlessly with existing grounders via domain predicates.
In extensive experiments on five benchmarks, applying diminutions selected by our strategy yields significant performance improvements, reducing grounding time by up to 70\% on average and decreasing the size of grounding files by up to 85\%. These results demonstrate that leveraging diminutions constitutes a robust and general-purpose approach for alleviating the grounding bottleneck in ASP.
\end{abstract}

\section{Introduction}

Answer Set Programming (ASP)~\cite{lifschitz2019answer} nowadays has become a prevalent tool for declarative problem solving.
With the help of its powerful expressiveness, practitioners can easily encode complex real-world problems in formal languages, e.g., robotics~\cite{erdem2011combining,zhu2023computing}, logistics~\cite{gebser2018experimenting}, and so on, and leave the solving process to efficient solvers, e.g., Clingo~\cite{gebser2019multi}, DLV~\cite{alviano2017asp}, Smodels~\cite{niemela2000smodels}, and ASSAT~\cite{lin2004assat}.
However, as a consensus in the literature, problem solving by ASP severely suffers from the so-called \emph{grounding bottleneck}~\cite{ostrowski2012asp,tran2023answer}, i.e., the program containing variables needs to be instantiated over its Herbrand universe before solving.
Imagine a domain of planning for household robots~\cite{puig2018virtualhome}, an ASP-based planner would fail in most cases even if there are only a few hundreds of objects, precisely because the grounding phase takes too long. 

Recent years have seen rapid progress in grounding techniques for addressing this issue, such as incremental grounding~\cite{gebser2019multi} (now integrated into Clingo) and lazy grounding~\cite{dal2009gasp}. In addition to these general-purpose methods, engineers often employ domain-specific heuristics.
For example, in \emph{Hamiltonian Circuit} problems, one may potentially replace $\texttt{edge/2.}$ in the following rule with $\texttt{relevantEdge/2}$,
\begin{verbatim}
{hc(X,Y): edge(X,Y)}=1 :- node(X).}
\end{verbatim}
plus some extra rules defining the latter, in order to focus the search merely on the edges of particular interest (\texttt{f/2.}),
\begin{verbatim}
relevantEdge(X,Y) :- edge(X,Y), f(X,Y).
\end{verbatim}
Although such exemplars are commonly seen in practice, how much redundancy they can actually save in the grounding phase still lacks investigation, potentially calling for a brand new theoretic framework.


To this end, we introduce a notion of controllable grounding, termed as \emph{diminution}.
Given a program (with variables), a diminution is a subset of its Herbrand universe that is used to ground the program.
One may desire certain properties (or say, restrictions) for such diminutions.
For example, each answer set of the program grounded under an \emph{admissible} diminution will be a subset of an answer set of the fully grounded program (i.e, the Herbrand instantiation).
It turns out that the concept of diminution is tightly related to \emph{splitting}~\cite{lifschitz1994splitting} and \emph{loop}~\cite{lin2004assat,gebser2005loops}.
We also provide some complexity results on deciding whether a set of constants is a diminution with certain properties.
To eventually present a more intuitive picture of the effectiveness of proper diminutions, we conduct a comprehensive empirical study, encompassing benchmarking problems lying in different levels of the \emph{Polynomial Hierarchy}~\cite{alviano2013fourth}.

The remainder of this paper is organized as follows.
We first review related work in this area.
We then introduce the necessary preliminaries on Answer Set Programming.
Finally, we adopt an application-oriented perspective, demonstrating the effectiveness of diminution on real-world problems and providing an experimental evaluation.

\section{Related Works}
\label{sec:related_work}

ASP solving can be viewed as comprising two phases: \emph{program instantiation} (also called \emph{grounding}) and \emph{answer set search}~\cite{faber2012intelligent}.
Numerous studies have aimed to accelerate ASP reasoning, with most focusing on the \emph{answer set search} phase. Prominent examples include \emph{splitting}~\cite{lifschitz1994splitting,ji2015splitting,ferraris2009symmetric}, \emph{forgetting}~\cite{lin1994forget,LIN2001143,lang2003propositional,eiter2019brief}, and \emph{conflict-driven} answer set solving~\cite{gebser2012conflict,lin2004assat}.)

Efforts have also been devoted to optimizing the grounding phase by leveraging techniques from \emph{(deductive) database} technology~\cite{ullman1988principles,apt1988towards,abiteboul1995foundations}. These techniques, such as \emph{top-down grounding}, \emph{bottom-up grounding}, and \emph{semi-naive grounding}—have been implemented in several widely-used grounders, notably \texttt{lparse}~\cite{syrjanen2000lparse,syrjanen2001omega}, \texttt{dlv}~\cite{leone2006dlv,faber2012intelligent,alviano2017asp}, and \texttt{gringo}~\cite{gebser2007gringo,gebser2011advances,gebser2019multi,gebser2022answer} to eliminate redundant computations and generate a semantically equivalent ground program substantially smaller than the full instantiation~\cite{kaufmann2016grounding}.
Several other techniques related to grounding have been extensively studied, including the \emph{magic set} method~\cite{bancilhon1985magic,beeri1987power,faber2007magic,ALVIANO2012156}, \emph{lazy grounding}~\cite{dal2009gasp,weinzierl2020advancing}, the use of \emph{dependency graphs} to determine grounding orders~\cite{faber2012intelligent,gebser2022answer}, and \emph{incremental grounding}~\cite{gebser2011incremental,gebser2019multi}.

Prior studies typically generate every complete answer set; by contrast, we explore an alternative pathway that significantly enhances grounding efficiency.

\section{Preliminaries}
We introduce disjunctive logic programs and then review the notions of loops and loop formulas.

\subsection{Basic Definitions}
Consider a first-order vocabulary $\mathcal{V} = \langle \mathcal{P}, \mathcal{C} \rangle$, where $\mathcal{P}$ and $\mathcal{C}$ are nonempty finite sets of predicates and constants, respectively.
Given a set $\mathcal{X}$ of variables, a \emph{term} is either a constant in $\mathcal{C}$ or a variable in $\mathcal{X}$.
An \emph{atom} is the form $p(t_1, \ldots, t_n)$ where $p\in\mathcal{P}$ and each $t_i$ ($1\leq i\leq n$) is a term. 
 A \emph{literal} is either an atom $a$ or its negation-as-failure literal $\mathit{not}~ a$.
A \emph{disjunctive logic program} (DLP) is a finite set of \emph{disjunctive rules} of the form
\begin{equation}  
\label{eq:dlp-rule}
    a_1 \lor \dots \lor a_k \gets a_{k+1}, \dots, a_m, \mathit{not}~a_{m+1}, \dots, \mathit{not}~a_n.
\end{equation}
where $0 \leq k \leq m \leq n$ and each $a_i$ is an atom. 
For a rule $r$ of the form~(\ref{eq:dlp-rule}), we define $\mathit{head}(r)=\{a_1,\dots,a_k\}$, $\mathit{body}^+(r)=\{a_{k+1},\dots,a_m\}$, $\mathit{body}^-(r)=\{a_{m+1},\dots,a_n\}$. We also let $\mathit{body}(r)=\mathit{body}^+(r)\cup\mathit{body}^-(r)$.
When convenient, we identify these sets with their corresponding propositional expressions
$\bigvee_{a\in\mathit{head}(r)} a$,
$\bigwedge_{a \in \mathit{body}^+(r)} a$,
and $\bigwedge_{a\in \mathit{body}^-(r)}\lnot\,a$.
We define $V(E)$ as the set of all variables and $C(E)$ as the set of all constants appearing in an expression $E$, where $E$ may be any expression in a DLP—such as an atom, a literal, a rule, or an entire program. Furthermore, we denote $\mathit{atom}(E)$ as the set of all atoms occurring in $E$, and $\mathit{pred}(E)$ as the set of all predicate symbols that appear in those atoms.

An expression $E$ is \emph{ground} iff $V(E) = \emptyset$. A rule $r$ is \emph{safe} iff 
$V(\mathit{head}(r) \cup \mathit{body}^-(r)) \subseteq V(\mathit{body}^+(r))$. 
When $k=1$, a rule of the form~(\ref{eq:dlp-rule}) is a \emph{normal rule}, and a finite set of normal rules is a \emph{normal logic program} (NLP).
A normal rule~$r$ is \emph{positive} if $\mathit{body}^-(r)=\emptyset$; a normal program is \emph{positive} iff all of its rules are positive.
A normal rule with an empty body is called a \emph{fact}; a rule with an empty head is called a \emph{constraint}.

A program $P$ is safe iff all of its rules are safe. Safety is typically ensured by introducing \emph{domain predicates} i.e., unary predicates whose ground instances enumerate the allowable constants for a variable. For example, to restrict a variable $X$ to range over $\{c_1,\dots,c_t\}$, include the atom $\mathtt{dom}(X)$ in the rule's positive body and add the facts $\mathtt{dom}(c_i).$ for each $c_i$ to the input program.

Since $P$ is function-free, its \emph{Herbrand universe} $HU(P)$ is the set of all constants occurring in $P$ (or a single fresh constant if none occur). Given a set $\mathcal{X}$ of variables and a set $\mathcal{D}$ of constants, a \emph{complete assignment} $\sigma:\mathcal{X}\to\mathcal{D}$ maps each variable to a constant.
For a rule $r$, write
\begin{equation*}
  r|_\mathcal{D}=\{\,r\sigma \mid \sigma:V(r)\to\mathcal{D}\,\}
\end{equation*}
for the set of all ground instances of $r$ over $\mathcal{D}$. For a program $P$, we construct a ground program by 
\begin{equation*}
    P|_{HU(P)} = \bigcup_{r\in P} r|_{HU(P)}.
\end{equation*}
Let $I$ be an interpretation, which is a set of ground atoms.
A ground rule $r$ is \emph{satisfied} by $I$, denoted  $I\models r$, iff either its body is false in $I$, or its body is true and at least one head atom belongs to $I$. formally,
\begin{equation*}
\begin{split}
I\models r \iff{} &\;
\neg\bigl(\mathit{body}^+(r)\subseteq I \wedge \mathit{body}^-(r)\cap I=\emptyset\bigr)\\
&\quad\lor\;\bigl(\mathit{head}(r)\cap I\neq\emptyset\bigr).
\end{split}
\end{equation*}
An interpretation $I$ is a \emph{model} of $P$ if it satisfies every ground rule in $P|_{HU(P)}$.
Answer sets are defined via the \emph{GL transformation}~\cite{gel91b}.
Given a DLP $P$ and a set $S$ of atoms, the \emph{GL transformation} of $P$ on $S$, written $P^S$, is obtained from $P|_{HU(P)}$ by deleting:  
\begin{enumerate}
    \item each rule that has $\mathit{not}~a$ in its body with $a \in S$, and
    \item all $\mathit{not}~a$ in the bodies of the remaining rules.
\end{enumerate}
$P^S$ is a ground program; that is, $\mathit{body}^-(r)=\emptyset$ for all $r\in P^S$.
Let $\Gamma(P^S)$ denote the set of $\subseteq$-minimal models of $P^S$.
A set $S$ of atoms is an \emph{answer set} of $P$ iff $S\in\Gamma(P^S)$. We write $AS(P)$ for the set of all answer sets of $P$.

The \emph{dependency graph} of $P$, $G_P=(V,E^+\cup E^-)$, is defined on the $P|_{HU(P)}$ by setting  $V$ as the set of ground atoms and adding an edge $(p,q)\in E^+$ whenever a rule has $q$ in its positive body and head $p$, and an edge $(p,q)\in E^-$ whenever $q$ appears negatively in a rule whose head is $q$. The positive dependency graph $G_P^+=(V,E^+)$.

A \emph{predicate-rule Graph} $G_{pr} = (V, E)$ for a logic program $P$ is defined by a node set $V = pred(P)\cup\{\,r \mid r \in P\}$
and an edge set $E$ containing directed edges of the form $(p/n , r)$ whenever the atom $p(t_1,\dots,t_n)$ occurs in the $body(r)$, and edges $(r, p/n)$ whenever the atom \(q(t_1,\dots,t_m)\) appears in the head of rule $r$.

\subsection{Loops and and Elementary Loops}

With the notions of loops and loop formulas~\cite{lin2004assat,lee05}, one can transform an ASP program $P$ into a propositional theory such that an interpretation is an answer set of $P$ if and only if it is a model of the theory. 
Note that we define loops and loop formulas on the ground program $P|_{HU(P)}$, as is done for the definition of answer sets. This differs slightly from the first-order loop formulas introduced in~\cite{chen2006first}.

Given a program $P$, a set $L$ of ground atoms is a \emph{loop} of $P$ if the subgraph of $G^+_P$ induced by $L$ is strongly connected. In particular, every singleton in $P|_{HU(P)}$ is  a loop of $P$.

For a loop $L$ of $P$, a ground rule $r\in P|_{HU{(P)}}$
is called an \emph{external support} of $L$ if 
$\mathit{head}(r) \in L$ and $L\cap \mathit{body}^+(r) = \emptyset$. We denote by $R^-(L, P)$ the set of all external support rules of $L$ in $P|_{HU{(P)}}$, $R^+(L, P) = \{ r\in P|_{HU(P)} \mid \mathit{head}(r) \in L\} \setminus R^-(L, P)$. The \emph{loop formula} of $L$ under $P$, written $LF(L, P)$, is the following implication
\begin{equation*}
    \bigwedge_{A\in L} A \supset \bigvee_{r\in R^-(L, P)}  \mathit{body}(r).
\end{equation*}
\begin{theorem}~\cite{lin2004assat}
\label{theorem:loop}
Given a program $P$ and an interpretation $I$. If $I$ is a model of $P$, then $I$ is an answer set of $P$ iff $I$ satisfies $LF(L, P)$ for all loops $L$ of $P$.
\end{theorem}

Then, we recall the notion of \emph{elementary loops} due to \cite{gebser2005loops}.
Let $P$ be a (ground) logic program and let $L\in loop(P)$.
$L$ is an \emph{elementary loop\/} of~$P$ iff
for every strict sub–loop $L'\subset L$ we have
\begin{equation*}
    R^-(L', P) \cap\; R^+(L, P) =\emptyset.
\end{equation*}
The set of all elementary loops of $P$ is denoted
$eloop(P)\subseteq loop(P)$. For elementary loops, loop formulas remain sufficient and necessary.
\begin{theorem}~\cite{gebser2005loops}
\label{thm:eloop}
For every ground program $P$ and interpretation~$I$, if $I$ is a model of $P$, then  $I$ is an answer set of $P$ iff $I$ satisfies $LF(eL, P)$ for all elementary loops $eL$ of $P$.
\end{theorem}

\section{Definitions and Properties of Diminution}
We formally introduce the notion of \emph{diminution} and investigate its properties. These definitions serve as the foundation for our acceleration techniques.

\subsection{Definitions of Diminution}
Grounding a program $P$ over $HU(P)$ can result in an exponential blow-up in the size of the grounded program.
A \emph{diminution} is a subset $D \subseteq HU(P)$ of constants such that grounding $P$ over $D$ yields the smaller program $P|_{D}$.Below, we define precisely when such a diminution is either \emph{admissible} or \emph{safe}.

\begin{definition}[Admissible Diminution]
Given a program $P$, a set of constants~$\mathcal{D} \subseteq HU(P)$ is called an \emph{admissible diminution} of $P$, if
for every answer set $I_\mathcal{D} \in AS(P|_{\mathcal{D}})$, there exists an answer set $I\in AS(P|_{HU(P)})$ such that $I_\mathcal{D}\subseteq I$.
\end{definition}

An admissible diminution $\mathcal{D}$ guarantees that every answer set of the ground program $P|_{\mathcal{D}}$ can be extended to at least one answer set of $P$. Next, we introduce the stronger notion of a \emph{safe diminution}.

\begin{definition}[Safe Diminution]
    Given a program $P$ and an \emph{admissible diminution} $\mathcal{D}$ of $P$, 
    we call $\mathcal{D}$ a \emph{safe diminution} of $P$ if, for every $I \in AS(P|_{HU(P)})$, there exists an answer set  $I_\mathcal{D} \in AS(P|_{\mathcal{D}})$ such that $I_\mathcal{D}\subseteq I$.
\end{definition}

For any program $P$ and any constant set $\mathcal{D}\subseteq HU(P)$, the following properties hold:
\begin{proposition}
For any program $P$:
    \begin{enumerate}
        \item $HU(P)$ itself is trivially a safe diminution of $P$.
        \item If $P$ has exactly one answer set(i.e., $\lvert AS(P)\rvert = 1$), then every admissible diminution $\mathcal{D}$ of $P$ is also safe.
        \item If $\lvert AS(P|_{\mathcal{D}})\rvert = 0$, then $\mathcal{D}$ is an admissible diminution of $P$; furthermore, if $AS(P|_{\mathcal{D}})=\{\emptyset\}$, then $\mathcal{D}$ is also a safe diminution of $P$.
    \end{enumerate}
\end{proposition}

A diminution $D\subseteq HU(P)$ that omits essential constants necessary to represent key elements of the problem may result in trivial solutions.
To prevent this, we require the diminution $D$ to ensure that every answer set of $P|_D$ contains preserved atoms formed from the predicates of a chosen predicate set $\mathcal{P}_{remain}$.

\begin{definition}[$\mathcal{P}_{remain}$-preserved diminution]
Given a program $P$ and a set $\mathcal{P}_{remain}$ of predicate symbols, an admissible (resp.\ safe) diminution $D\subseteq HU(P)$ is \emph{$\mathcal{P}$-preserved admissible} (resp.\ \emph{$\mathcal{P}$-preserved safe}) if for every $I_{\mathcal{D}}\in AS(P|_{\mathcal{D}})$, there exists $I\in AS(P|_{HU(P)})$ such that
$\{a\mid a\in I,\, pred(a)\in \mathcal{P}_{remain}\} = \{a\mid a\in I_{\mathcal{D}},\,pred(a)\in \mathcal{P}_{remain}\}$.
\end{definition}

$P_{\mathrm{remain}}$-preservation means that reducing $HU(P)$ to $D$ never omits any atoms formed by $P_{\mathrm{remain}}$, preserving all essential facts. The following example, drawn from a basic graph coloring domain, provides an intuitive understanding of our definition..

\begin{example}[Graph Coloring Problem]
\label{example:1}Let $P_1$ be the following ASP program for the \emph{3-graph coloring problem} shows in Figure~\ref{fig:combined}(a):
\begin{small}
\begin{verbatim}
arc(1,2). arc(1,3). arc(2,3). arc(3,5). 
arc(3,6). arc(5,6). arc(4,5). arc(4,8). 
arc(5,8). arc(6,7). arc(6,9). arc(7,9).
col(r). col(b). col(g).

color(V,C):-vertex(V),col(C),
             not othercolor(V,C).
othercolor(V,C):-vertex(V),col(C),col(C1),
                  C != C1,color(V,C1).
:-arc(V1,V2),col(C),color(V1,C),color(V2,C).
\end{verbatim}
\end{small}
\end{example} 

\begin{example}[Safe Diminution of $P_1$]
\label{exp:safe}
Consider the graph shows in Figure~\ref{fig:combined}(a), define $\mathcal{D}_1=\{\texttt{1},\texttt{2},\texttt{3}\}\cup\{\texttt{r},\texttt{b},\texttt{g}\}$. 
One can verify that $\mathcal{D}_1$ is indeed a safe diminution, meaning that every answer set in $AS(P_1|_{\mathcal{D}_1})$ extends to some answer set in $AS(P_1|_{HU(P_1)})$ and, conversely, each answer set in $AS(P_1|_{HU(P_1)})$ restrict to some answer set in $AS(P_1|_{\mathcal{D}_1})$. 
However, since  answer sets of $P_1|_{\mathcal{D}_1}$ cannot assign colors to all nodes in the original graph; thus, $\mathcal{D}_1$ is not a $\{\texttt{color/2}\}$-preserved safe diminution, it is indeed $\{\texttt{arc/2}, \texttt{col/1}\}$-preserved safe diminution. 
However, in practice, preserving predicates such as $\texttt{color/2}$ is of primary importance.
\end{example}

\begin{example}[Admissible but Unsafe Diminution of $P_1$]
\label{example:admissible}
Consider the graph shown in Figure~\ref{fig:combined}(b), define $\mathcal{D}_1=\{\texttt{1},\texttt{2},\texttt{5},\texttt{7}\}\cup\{\texttt{b},\texttt{r}\}$.  One finds $AS(P_1|_{\mathcal{D}_1})$=
{\small 
\begin{align*}
\begin{array}{l}
\{\{\texttt{color(1,b)}, \texttt{color(2,r)}, 
  \texttt{color(5,r)}, \texttt{color(2,b)}\}, \\
\{\texttt{color(1,r)}, \texttt{color(2,b)}, 
 \texttt{color(5,b)}, \texttt{color(2,r)}\}\}
\end{array}
\end{align*}}
Every $I_{\mathcal{D}_1}\in AS(P_1|_{\mathcal{D}_1})$ can be extended to a full 3-coloring of $P_1$, however, none of these sets contain the atom \texttt{color(5,g)}. However, the atom \texttt{color(5,g)} does appear in some answer set of $P_1$. This counterexample demonstrates that $\mathcal{D}_1$ is admissible but not safe. As in the previous example, this is also not a $\{\texttt{color/2}\}$-preserved admissible diminution.
\end{example}

\begin{example}[Non-Admissible Diminution of $P_1$]
\label{example:non-admissible}
Define $\mathcal{D}_{2}=HU(P_{1})\setminus\{ \texttt{1} \}$. One partial answer set of $P_{1}|_{\mathcal{D}_{2}}$ is shown in Figure~\ref{fig:combined}(c).
Here nodes $\texttt{2}$, $\texttt{3}$, and $\texttt{5}$ use all three colors $\texttt{r}$, $\texttt{g}$, $\texttt{b}$, leaving no available color for node $\texttt{1}$ without conflict. Therefore, this partial answer set cannot extend to a full 3-coloring of $P_{1}$, making $\mathcal{D}_{2}$ a non-admissible diminution.
\end{example}

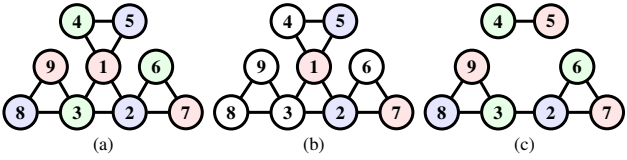
\begin{figure}[t]
  \centering
  \begin{tikzpicture}[
      scale=0.7,
      transform shape,
      every node/.style={
        circle, draw, line width=0.4mm,
        minimum size=3mm,    
        node distance=1cm,
        font=\bfseries
      }
    ]
    \begin{scope}[xshift=0cm]
      \node[fill=red!10]   (1a) at (0,0)         {1};
      \node[fill=blue!10]  (2a) at (300:1cm)     {2};
      \node[fill=green!10] (3a) at (240:1cm)     {3};
      \node[fill=green!10] (4a) at (120:1cm)     {4};
      \node[fill=blue!10]  (5a) at (60:1cm)      {5};
      \node[fill=green!10] (6a) at (0:1cm)       {6};
      \node[fill=red!10]   (7a) [right=0.4cm of 2a] {7};
      \node[fill=blue!10]  (8a) [left=0.4cm of 3a]  {8};
      \node[fill=red!10]   (9a) at (180:1cm)     {9};
      \foreach \u/\v in {
        1a/2a,1a/3a,2a/3a,1a/4a,1a/5a,4a/5a,
        2a/6a,2a/7a,6a/7a,3a/8a,3a/9a,8a/9a}
        \draw[line width=0.4mm] (\u) -- (\v);
      \node[shape=rectangle,draw=none,inner sep=0pt,font=\normalsize]
        at (0,-1.5) {(a)};
    \end{scope}

    \begin{scope}[xshift=4cm]
      \node[fill=red!10]   (1b) at (0,0)         {1};
      \node[fill=blue!10]  (2b) at (300:1cm)     {2};
      \node[]              (3b) at (240:1cm)     {3};
      \node[]              (4b) at (120:1cm)     {4};
      \node[fill=blue!10]  (5b) at (60:1cm)      {5};
      \node[]              (6b) at (0:1cm)       {6};
      \node[fill=red!10]   (7b) [right=0.4cm of 2b] {7};
      \node[]              (8b) [left=0.4cm of 3b]  {8};
      \node[]              (9b) at (180:1cm)     {9};
      \foreach \u/\v in {
        1b/2b,1b/3b,2b/3b,1b/4b,1b/5b,4b/5b,
        2b/6b,2b/7b,6b/7b,3b/8b,3b/9b,8b/9b}
        \draw[line width=0.4mm] (\u) -- (\v);
      \node[shape=rectangle,draw=none,inner sep=0pt,font=\normalsize]
        at (0,-1.5) {(b)};
    \end{scope}

    \begin{scope}[xshift=8cm]
      \node[fill=blue!10]  (2c) at (300:1cm)     {2};
      \node[fill=green!10] (3c) at (240:1cm)     {3};
      \node[fill=green!10] (4c) at (120:1cm)     {4};
      \node[fill=red!10]   (5c) at (60:1cm)      {5};
      \node[fill=green!10] (6c) at (0:1cm)       {6};
      \node[fill=red!10]   (7c) [right=0.4cm of 2c] {7};
      \node[fill=blue!10]  (8c) [left=0.4cm of 3c]  {8};
      \node[fill=red!10]   (9c) at (180:1cm)     {9};
      \foreach \u/\v in {
        2c/3c,4c/5c,2c/6c,2c/7c,6c/7c,3c/8c,3c/9c,8c/9c}
        \draw[line width=0.4mm] (\u) -- (\v);
      \node[shape=rectangle,draw=none,inner sep=0pt,font=\normalsize]
        at (0,-1.5) {(c)};
    \end{scope}

  \end{tikzpicture}
  \caption{(a) Complete 3‑coloring for Example~\ref{example:1}. 
           (b) Partial 3‑coloring for Example~\ref{example:admissible}.  
           (c) Partial 3‑coloring for Example~\ref{example:non-admissible}.}
  \label{fig:combined}
\end{figure}

\subsection{Properties of Diminution}

We now examine diminutions in greater detail. First, we present properties describing how limiting the constant set impacts a program’s answer sets, and then we introduce a loop-based decision procedure to test whether a given constant set qualifies as a diminution.

\begin{proposition}
Given a positive program $P$, then every $\mathcal{D}\subseteq HU(P)$ are safe-diminutions of $P$.
\end{proposition}
\begin{proof}
Since $P|_{\mathcal{D}}\subseteq P|_{HU(P)}$ and both are positive, then
$\mathit{LM}\bigl(P|_{\mathcal{D}}\bigr)\subseteq\mathit{LM}(P|_{HU(P)})$ ~\cite{janhunen2007automated}, where $LM(P)$ denotes the least model of $P$.
Moreover, for any positive program $P$, its unique answer set coincides with its least model, it follows that the unique answer set $I_{\mathcal{D}}\in \mathit{AS}(P|_{\mathcal{D}})$ satisfies 
$I_{\mathcal{D}}\;\subseteq\;I_{\mathrm{full}}\in \mathit{AS}(P|_{HU(P)})$, as required.
\end{proof}

Furthermore, we indentify a class of diminutions related to the \emph{splitting set theorem}~\cite{lifschitz1994splitting}.

\begin{definition}[Splitting-Safe Diminution]
Given a program $P$ there exists at least one answer set, and a subset of constants $\mathcal{D}\subseteq HU(P)$, we call $\mathcal{D}$ a \emph{splitting-safe diminution} if the set of all ground atoms in $P|_{\mathcal{D}}$ constitutes a splitting set of $P|_{HU(P)}$.
\end{definition}

\begin{theorem}
    Given a subset~$\mathcal{D}\subseteq HU(P)$ for an ASP program~$P$, if $\mathcal{D}$ is a splitting-safe diminution of $P$, then $\mathcal{D}$ is a safe diminution of $P$.
\end{theorem}
\begin{myproof}
Since the $atom(P|_{\mathcal D})$ is a splitting set of the $P|_{HU(P)}$, therefore, $P|_{\mathcal D}$ serves as the \emph{bottom} $B$ of $P|_{HU(P)}$. Define $T = P|_{HU(P})\setminus B$. The splitting theorem guarantees that each answer set of $B$ extends-with atoms from $T$-to an answer set of the $P|_{HU(P)}$, and each answer set of the $P|_{HU(P)}$ restricts to one of $B$. Therefore $\mathcal{D}$ is a safe diminution of $P$.
\end{myproof}

Verifying $\mathcal{D}$ is a splitting set can by a linear scan of the ground program, yet it still be overly restrictive. We therefore define a simple syntactic class of programs such that, for any program $P$ in this class and any $\mathcal{D}\subseteq HU(P)$ is a safe diminution.

\begin{definition}[Term‑Preserved Program]\label{def:term_preserve}
A normal rule $r$ is \emph{term‑preserved} if $C(body(r)) \subseteq C(head(r))$ and $V(body(r)) \subseteq V(head(r))$.
A normal program is \emph{term‑preserved} when all its rules are term‑preserved.
\end{definition}

\begin{example}[Triangle detection, term-preserved program]\label{ex:triangle}
Consider the following program, which determines which triples in a given edges form triangles.
\begin{small}
\begin{verbatim}
edge(a,b). edge(b,c). edge(c,a).
tri(X,Y,Z):-edge(X,Y),edge(Y,Z),edge(Z,X).
\end{verbatim}
\end{small}
Each body variable $X,Y,Z$ reappears in the head, and the facts have empty
bodies, so the entire program is term-preserved.
\end{example}

To construct a term-preserved program that ensures every $\mathcal{D}\subseteq HU(P)$ is a safe diminution, we apply the Domain Predicate Lifting procedure which is detailed in Appendix~A.1. In brief
this procedure transforms $P$ into $P^\uparrow$, where for each constant $c\in C(P)$, it introduces a fresh variable $v_c$ and a domain predicate $p_c$, rewriting each occurrence of $c$ by $v_c$ guarded with $p_c(v_c)$, and adding the fact $p_c(c)$.  As a result, for any $\mathcal{D}\subseteq HU(P)$, each answer set of $P^\uparrow|_{\mathcal{D}}$ is obtained by extending some $I\in AS(P|_{\mathcal{D}})$ with the facts $\{p_c(c)\mid c\in C(P)\}$. Because the transformation does not introduce new constants, $HU(P^\uparrow)=HU(P)$.

\begin{theorem}
\label{thm:D_preserve}
Let $P$ be a term‑preserved program there exists at least one answer set, every $D\subseteq HU(P)$ is a safe diminution of $P$.
\end{theorem}
\begin{myproof}
Write $P^{\uparrow}|_D=\mathcal F^{\neg c}\cup\mathcal F^{c}\cup P_1$
and $P^{\uparrow}|_{D'}=\mathcal F^{\neg c}\cup\mathcal F^{c}\cup P_1\cup P_2$,
where $\mathcal F$ is the fact set of $P^{\uparrow}|_D$,  
$\mathcal F^{\neg c}=\{\,f\in\mathcal F \mid c \text{ does not occur in }f\,\}$,
and $\mathcal F^{c}=\mathcal F\setminus\mathcal F^{\neg c}$.
All atoms in $P_1\cup\mathcal F^{\neg c}$ omit the constant $c$, 
whereas every rule head in $P_2\cup\mathcal F^{c}$ contains $c$.  
Hence  
$U=\mathrm{atoms}(P_1)\cup\mathcal F^{\neg c}$,
which contains no atom mentioning $c$, is a splitting set of
$P^{\uparrow}|_{D'}$.
Because every fact in $\mathcal F$ is present in every answer set,
each $I'\in AS(P^{\uparrow}|_{D'})$ extends some
$I\cup\mathcal F^{c}\in AS(P^{\uparrow}|_D)$, and conversely
every $I\cup\mathcal F^{c}\in AS(P^{\uparrow}|_D)$ extends to
an $I'\in AS(P^{\uparrow}|_{D'})$.  
Therefore $\mathcal D$ is a safe diminution.
\end{myproof}

We use \emph{strong equivalence}~\cite{lin2002reducing,turner2003strong}
to identify program rewriting that preserve answer sets in every context. Programs $P$ and $Q$ are strongly equivalent, for any program $R$, the unions $P\cup R$ and $Q\cup R$ have identical answer sets.
\begin{proposition}
\label{prop:adm-not-transfer}
Let $P_{1}$ and $P_{2}$ be programs such that
$P_{1}|_{HU(P_{1})}$ and $P_{2}|_{HU(P_{2})}$ are
\emph{strongly equivalent}, an admissible diminution for $P_{1}$ need not be an admissible
diminution for $P_{2}$.
\end{proposition}
\begin{proof}
Consider two programs that become identical once fully grounded.
Let $P_1$ be
\begin{small}
\begin{verbatim}
p(a). p(b). r(a) :- p(a). r(b) :- p(b).
r(c) :- p(c). r(c) :- not r(b).
\end{verbatim}
\end{small}
and let $P_2$ be
\begin{small}
\begin{verbatim}
p(a). p(b). r(X) :- p(X). r(c) :- not r(b).
\end{verbatim}
\end{small}

Fix the restricted constant set
$\mathcal{D}=\{\texttt{a}\}\subseteq HU(P_1)=HU(P_2)=\{\texttt{a,b,c}\}$.
$P_1|_\mathcal{D}$ yielding the single answer set
$I=\{\texttt{p(a)},\texttt{p(b)},\texttt{r(a)},\texttt{r(b)}\}\in AS(P_1)$. Hence $\mathcal{D}$ is an admissible diminution of $P_1$. $P_2|_\mathcal{D}$ yielding the single answer set
$J=\{\texttt{p(a)},\texttt{p(b)},\texttt{r(a)},\texttt{r(c)}\}$. 
No answer set of the $P_2$ can contain \texttt{r(c)}, therefore $\mathcal{D}$ is \emph{not} an admissible diminution of $P_2$.
\end{proof}

Then, we present the computational complexity of deciding an admissible or safe diminution.
\begin{theorem}
\label{theoram:complexity}
Given a program $P$ and a subset~$\mathcal{D}$ of $HU(P)$, deciding whether $\mathcal{D}$ is an admissible diminution of $P$ is $\mathsf{coNP}$-hard; deciding whether $\mathcal{D}$ is a safe diminution of $P$ is $\mathsf{coNP}$-hard.
\end{theorem}

\begin{myproof}
We can construct a problem~$P$ from a 3-SAT problem by adding $\{ a \gets \mathit{not}~a'. \ a'\gets \mathit{not}~a.\}$ for each atom $a$, without loss of generality, for each clause $\neg a \lor b \lor \neg c$ adding $\{ \gets a,\, \mathit{not}~ b,\, c.\}$, and adding $\{f(o_1).\ f(o_2).\ \gets f(x), f(y), x\neq y. \}$. $HU(P) = \{o_1, o_2\}$, $P|_{HU(P)}$ has no answer sets, $P|_{\{o_1\}}$ has an answer set if{f} the 3-SAT problem is satisfiable.
Then $\{o_1\}$ is an admissible or a safe diminution if and only if the 3-SAT problem is unsatisfiable.
\end{myproof}

By Theorem~\ref{theoram:complexity}, identifying admissible and safe diminutions in a given programs is computationally challenging. In many applications, our goal is not to enumerate all answer sets, but to obtain at least one solution that is practically useful.

With the help of the notions of loops and loop formulas, we can provide a sufficient condition for \emph{admissible diminution}.

\begin{definition}[Loop-Admissible Diminution]
    Let $P$ be a program (with variables) and $\mathcal{D} \subseteq HU(P)$. We call $\mathcal{D}$ a \emph{loop-admissible diminution} of $P$ if
    \begin{enumerate}
        \item for every answer set $I_{\mathcal{D}}$ of $P|_{\mathcal{D}}$, there exists an interpretation $I'$ such that $I_{\mathcal{D}} \cup I'$ satisfies rules in $P|_{HU(P)}$ and loop formulas for every loops $L'$ of $P|_{HU(P)}$ with $L'\subseteq I'$, and 
        \item there does not exist a loop $L$ of $P|_{HU(P)}$ such that $L$ is not a loop of $P|_{\mathcal{D}}$ and $L$ contains a loop $L'$ of $P|_{\mathcal{D}}$ with $R^-(L', P|_{\mathcal{D}}) \neq \emptyset$.
    \end{enumerate}
\end{definition}
Intuitively, we require the rules in $P|_{HU(P)} \setminus P|_{\mathcal{D}}$ and newly introduced loop formulas can be satisfied by expanding $I_{\mathcal{D}}$ with some $I'$.

\begin{theorem}
\label{theorem:loop_admissible}
Given a subset~$\mathcal{D}$ of $HU(P)$ for an ASP program~$P$, if $\mathcal{D}$ is a loop-admissible diminution of $P$, then $\mathcal{D}$ is an admissible diminution of $P$.
\end{theorem}
\begin{myproof}
The argument mirrors that of Theorem~\ref{thm:elementaryLoopAdmissible}:
replace the \emph{elementary loop} with \emph{loop} and
observe that theorem~\ref{theorem:loop_admissible} still holds.
\end{myproof}
To reduce the number of loops required when deciding whether a diminution is admissible, we introduce the notion of an \emph{elementary-loop-admissible diminution}.
\begin{definition}[Elementary-Loop-Admissible Diminution]
\label{def:el ad diminution}
Let $P$ be a program(with variables) and $\mathcal{D}\subseteq HU(P)$. we call $\mathcal{D}$ an elementary-loop-admissible diminution of $P$ if
\begin{enumerate}
    \item for every answer set $I_\mathcal{D}$ of $P|_\mathcal{D}$, there exist an interpretation $I'$ such that $I_\mathcal{D}\cup I'$ is a \emph{model} of $P|_{HU(P)}$ and 

    \item for every elementary loops $eL$ of $P|_{HU(P)}$ such that $eL\subseteq I'$, the interpretation $I_{\mathcal{D}}\cup I'$ satisfies $LF(eL, P)$
    
    \item No elementary loop $eL$ of $P|_{HU(P)}$ that is not an elementary loop of $P|_\mathcal{D}$ and there exist elementary loops $eL'$ of $P|_{\mathcal{D}}$ such that $R^-(eL', P|_\mathcal{D}) \neq \emptyset$.
\end{enumerate}
\end{definition}
\begin{theorem}
\label{thm:elementaryLoopAdmissible}
Given a subset $\mathcal{D}$ of $HU(P)$ for given program $P$, if $\mathcal{D}$ is an elementary-loop-admissible reduction of $P$, then $\mathcal{D}$ is an admissible reduction of $P$.
\end{theorem}
\begin{proof}
Let $I_\mathcal{D}$ be an answer set of $P|_{\mathcal{D}}$, from the definition of elementary-loop-admissible diminution, there exists the set $I'$ of ground atoms such that $I_\mathcal{D}\cup I'$ is a supported model of $P$ and $I_\mathcal{D}\cup I'$ satisfies loop formulas of elementary loops $eL$ of $P$ with $eL\subseteq I_\mathcal{D}$ or $eL\subseteq I'$.

To prove that $\mathcal{D}$ is an admissible diminution of $P|_{HU(P)}$, we need to show that $I_\mathcal{D} \cup I'$ is an answer set of $P$. By Theorem~\ref{thm:eloop}, it suffices to prove that $I_\mathcal{D} \cup I'$ satisfies $LF(eL, P)$ for every elementary loop $eL$ of $P|_{HU(P)}$. Consider an arbitrary elementary loop $eL$ of $P$. We analyze all possible cases:
\begin{enumerate}
    \item Case $eL \not\subseteq I_{\mathcal{D}} \cup I'$: by the definition of loop formula, $\neg\bigwedge_{a\in eL} a$ holds, hence $LF(eL, P|_{HU(P)})$ trivially true.
    \item Case $eL \subseteq I_\mathcal{D}$: Since $I_\mathcal{D}$ is an answer set of $P|_\mathcal{D}$, $I_\mathcal{D}$ satisfies some rules $r\in R^-(eL, P|_\mathcal{D})\subseteq R^-(eL, P|_{HU(P)})$, hence $LF(eL, P|_{HU(P)}$ satisfied by $ I_{\mathcal{D}} \cup I'$.
    \item Case $eL \subseteq I'$: by condition 1 of elementary-loop-admissible diminution, $I'$ satisfies $LF(eL, P|_{HU(P)})$, hence $I'\cup I_\mathcal{D}$ satisfies $LF(eL, P|_{HU(P)})$.
    \item Case $eL \cap I_{\mathcal{D}} \neq \emptyset \wedge eL \cap I_{HU(P)} \neq \emptyset \wedge el\subseteq I_\mathcal{D}\cup I|_{HU(P)}$: suppose such $eL$ exists. Since $I_\mathcal{D}$ is an answer set of $P|_\mathcal{D}$,  by Theorem~\ref{thm:eloop}, there must exist a loop $eL' \subseteq eL \cap I_\mathcal{D} \subseteq  eL$ However, this contradicts condition 3 of the definition~\ref{def:el ad diminution}.
\end{enumerate}
\end{proof}
Because elementary loops form a subset of all loops, any loop-admissible diminution is automatically elementary-loop-admissible. The converse does not necessarily hold, so the elementary notion is strictly weaker.

\section{Implementation and Evaluation}

\begin{table*}[!t]
  \centering
  \small
  \setlength{\tabcolsep}{1.6mm}
  \begin{tabular*}{0.9\textwidth}{@{\extracolsep{\fill}} l c r r r r r r}
    \toprule
    \textbf{Domain} & \textbf{$\mathcal{D}$}
      & \makecell[c]{Grounding (s)}
      & \makecell[c]{Final Size (MB)}
      & \makecell[c]{Solving (s)}
      & \makecell[c]{Timeout (\%)}
      & \makecell[c]{$\Delta$Size (MB/step)}
      & \makecell[c]{Avg. Steps} \\
    \midrule
\texttt{Clingo} \\
\multirow{2}{*}{VH}
  & $Heu$            & 13.60   & 40.16  & 2.23  & 4.91  & 2.03  & 17.17 \\
  & $\mathrm{HU}$    & 144.42 & 463.68 & 7.23 & 61.65 & 50.86 & 11.56 \\
\addlinespace
\multirow{2}{*}{AWS}
  & $Heu$          & 2.26 &   5.82 & 4.07 & 5.00  & 0.26  & 19.77 \\
  & $\mathrm{HU}$  & 56.89   & 538.10 & 19.02 & 40.80 & 25.01 & 20.32 \\
\addlinespace
\multirow{2}{*}{GW}
  & $Heu$           & 41.89  & 108.11 & 1.20 & 0.00 & 2.10  & 47.74, \\
  & $\mathrm{HU}$   & 182.61   & 187.19 & 4.83 & 11.00 & 3.28 & 47.15 \\
\addlinespace
\multirow{2}{*}{HC}
  & $Heu$         &  1.09   &   7.42 & 0.69  & 0.00  & ---   & ---   \\
  & $\mathrm{HU}$ &  69.19   & 759.26 & 20.45 & 0.00  & ---   & ---  \\
\addlinespace
\multirow{2}{*}{SM}
  & $Heu$           & 0.19   &   1.76 & 0.01  & 0.00  & ---   & ---   \\
  & $\mathrm{HU}$   &  35.83   & 402.91 & 9.54  & 0.00  & ---   & ---   \\
\addlinespace
\midrule 
\texttt{Dlv2} \\
\multirow{2}{*}{VH}
  & $Heu$ &  32.64 &  99.16  & 11.77 & 14.90 &  6.78  & 16.85 \\
  & $\mathrm{HU}$  &  41.29 & 524.69  & 1.73 & 96.67 & 298.16 &  3.84 \\
\addlinespace
\multirow{2}{*}{GW}
  & $Heu$          & 106.06 &52.01 & 48.17 & 29.00 & 0.88 & 47.74\\
  & $\mathrm{HU}$  &110.31 &125.05 &107.14 &84.00  & 5.00 & 28.64 \\
\addlinespace
\multirow{2}{*}{HC}
  & $Heu$          &  0.31 &   6.32 & 0.01  &  0.00 & ---   & ---   \\
  & $\mathrm{HU}$  &  30.15 & 667.10 & 67.04 &  0.00 & ---   & ---   \\
\addlinespace
\multirow{2}{*}{SM}
  & $Heu$           &  0.28 &   1.47 & 0.14  &  0.00 & ---   & ---   \\
  & $\mathrm{HU}$   &  0.515&  38.42 & 0.865 & 50.00 & ---   & ---   \\
\bottomrule
  \end{tabular*}
  \caption{Benchmark results for \texttt{Clingo} and \texttt{DLV2}.
  Domains: VH = \textbf{VirtualHome}, AWS = \textbf{AutomatedWarehouse}, GW = \textbf{2DGridWorld}, HC = \textbf{HamiltonianCircuit},
  SM = \textbf{StableMarriage}.  $\mathcal{D}$: \emph{Heu} denotes the
  heuristic diminution, \emph{HU} the full Herbrand universe. The Columns
  report mean grounding time, final ground size, solving time, timeout
  rate, average per-step size growth, and average number of steps
  (--- indicates not applicable), respectively.}
  \label{table:result}
\end{table*}

We present the practical implementation of \emph{diminution} and evaluate its impact on solving efficiency.
We describe our heuristic for selecting a diminution, show how it is enforced in standard grounders via domain predicates, and report experiments that quantify the resulting speed-ups.

\subsection{Implementation}
We begin by showing how diminution can be simulated with domain predicates in a standard \emph{bottom-up} grounding workflow, Given a program $P$, we build its predicate–rule dependency graph, compute the graph’s strongly connected components (SCCs), and order these components topologically. The resulting bottom-up grounding workflow is executed by the procedure \textsc{Grounding}($P$) presented  in~\cite{gebser2022answer} This classic algorithm is reproduced in full in Appendix~A.2..

\begin{definition}[$\mathcal{D}$-Guarded Program]
Let $P$ be a program and let $\mathcal{D}\subseteq HU(P)$. 
We construct $P^{[\mathcal{D}]}$ from $P$ by 
\begin{enumerate}
    \item Adding atoms in form of $\texttt{dom(X)}$ in $body^+(r)$ for some $r\in P$;
    \item Adding rules $r$ such that $pred(head(r)) = \{\texttt{dom/1}\}$.
\end{enumerate}
where $\texttt{dom/1}$ is a domain predicate and may only be instantiated with constants from $\mathcal{D}$. 

Let $C_1 \prec C_2 \prec\cdots\prec C_n$ be a topological ordering of the SCCs of the $G_{pr}$ of $P^{[\mathcal{D}]}$. Then $P^{[\mathcal{D}]}$ is called a \emph{$\mathcal{D}$-guarded program} if the following conditions hold for every component SCCs orderings of predicate–rule dependency graph of $P^{[\mathcal{D}]}$:
\begin{enumerate}
\item No $C_i$ contains both $\texttt{dom/1}$ and $p \in pred(P)$. 

\item If $\texttt{dom/1} \in pred(C_i)$, then no $C_j \prec C_i$ contains a rule $r$ with $V(r) \not= \emptyset$ and $head(r)$ not contain $\texttt{dom/1}$.  

\item Let $t =\max\{i \mid \texttt{dom/1} \in pred(C_i)\}$.
  Then for every rule $r\in C_j, j\ge i$ such that who contains $p\in \{p' \mid p' \in pred(C_t), p' \not= \texttt{dom/1},i>t\}$, there exist a $\texttt{dom(X)} \in body^+(r)$ for every variable $X$ in $p(\vec t) \in atom(r)$.
\end{enumerate}
\end{definition}
Intuitively, the above conditions guarantee that for every atom $a$ in the original program $P$ that contains a variable (w.l.o.g.\ call it $X$):
\begin{enumerate}
    \item If $pred(a)$ is grounded before \texttt{dom/1}, a literal $\texttt{dom}(X)$ is introduced, forcing $X$ to be instantiated only with constants in $\mathcal{D}$.
    \item For predicates grounded after \texttt{dom/1}, the grounding of their variables must respect the constant range already fixed by the instantiation of \texttt{dom/1}.
\end{enumerate}
\begin{theorem}
    Given a program $P$ and its $\mathcal{D}$-guarded program $P^\mathcal{D}$, let
    \[ \mathcal{F}_{dom}=\{f\in P \mid pred(f) = \texttt{dom/1}\}.
    \]
    Then $AS(\texttt{Grounding}(P^\mathcal{D})\setminus \mathcal{F}_{dom})=AS(P|_\mathcal{D})$.
\end{theorem}
This conclusion can be derived step by step using $\textsc{Grounding}($P$)$. We postpone the detailed proof to Appendix B.1.

The definition and theorem show that inserting a domain predicate in the prescribed way simulate grounding using any given $\mathcal{D}$. This makes diminution usable with any ASP Solver. 

\subsection{Benchmarks}
Our experiments comprise two parts. The first part considers three optimization problems: \textbf{VirtualHome} (VH, high-level household robotic planning)~\cite{puig2018virtualhome}, \textbf{AutomatedWarehouse} (AWS, multiple robots picking up and delivering product while avoiding collisions with one another)~\cite{gebser2018experimenting}, and \textbf{2DGridworld} (GW, single-robot path finding with static obstacles)~\cite{mcdermott20001998}. We follow the convention of using incremental grounding~\cite{gebser2019multi} to solve these optimization problems.
For the \textit{second} part, we study two classic satisfiability problems from past ASP compititions~\cite{alviano2013fourth}: \textbf{HamiltonianCircuit} (HC) and \textbf{StableMarriage} (SM).

We design the following heuristics, which can be encoded via the domain-predicate injection method introduced earlier and thus serve as concrete diminutions in the grounding process:
\begin{enumerate}
    \item For \textbf{VH}, a skeleton plan (i.e., a course of intermediate actions or states) is generated by LLMs, as done by some contemporary work~\cite{lin2024clmaspcouplinglargelanguage}, We then propagate from the given skeleton plan and facts to build the set of relevant constants, treating it as the domain’s diminution.
    \item For \textbf{AWS} and \textbf{GW}, we restrict locations to grid cells whose rows or columns align with the boundaries of objects of interest (e.g., obstacles, shelves, and other relevant items).
    \item For \textbf{HC}, the selection of the next node is restricted within a neighborhood of the current node.
    \item For \textbf{SM}, each men only proposes to the women associated with indices in a predefined range.
\end{enumerate}
Note that adding these heuristics may, in general, render the program unsatisfiable. for each domain, we  generate the test instances on purpose so that at least one solution remains under the given heuristic. Noting that above heuristics may causes the solutions without the parts we require, we therefore ensure in our encoding that the heuristic‐derived diminution is $\mathcal{P}$-preserved. In practice, $\mathcal{P}$ comprises the predicates of interest in a problem’s solutions (e.g.\ $\{\texttt{occurs/2}\}$ for VH, $\{\texttt{hc/2}\}$ for HC, $\{\texttt{match/2}\}$ for SM).

After introducing the above ad-hoc heuristics, we present a more general methodology to guide the construct of such heuristics.

Given a constraint‐satisfaction problem, let a feasible solution be a mapping
$f : \Phi \!\to\! \Psi$ from variables~($\Phi$) to values~($\Psi$).
Assume an oracle $\hat f$ that proposes a (possibly infeasible) guess.
We distinguish three oracle modes:
\begin{enumerate}
  \item $\hat f_1 : \hat\Phi \to \Psi$ with $\hat\Phi \subseteq \Phi$
        \quad—provides a \emph{partial assignment}.
  \item $\hat f_2 : \Phi \to \hat\Psi$ with $\hat\Psi \subseteq \Psi$
        \quad—\;restricts the \emph{value set}.
  \item $\hat f_3 : \Phi \to \bigcup_{\phi \in \Phi}\mathcal{N}\!\bigl(f'(\phi)\bigr)$,
        where $f' : \Phi \to \Psi$ is a full guess and
        $\mathcal{N}(\psi)$ denotes a bounded neighborhood of~$\psi$
        \quad—limits the search to a predefined \emph{neighborhood}.
\end{enumerate}
Each oracle mode yields a candidate constant set that can serve as a diminution~$\mathcal{D}$.

\subsection{Experimental Results}
All experiments were run on a Windows PC with an AMD Ryzen 5 9700X  
(6 cores, 3.8 GHz) and 47.2 GB DDR5 RAM.  
Clingo (Python API 5.8.0) was configured to return a single answer set, while \texttt{DLV 2.1.2‑win64} was invoked with
\texttt{--stats=2}, which enumerates all answer sets in order
to report aggregate statistics.  

In our evaluation (Table~\ref{table:result}) we measure, for each domain and its diminution variant, the \emph{Grounding} and \emph{Solving} times~(s), the \emph{Final Size} of the ground program (MB), and the \emph{Timeout} rate~(\%).  
For domains run under \texttt{incmode} we additionally report the average number of executed steps (\emph{Avg. Steps}) and the average per-step size growth $\Delta$\emph{Size}, which gauges the extra cost of each incremental grounding round.

In \texttt{incmode} terminate the process once the current step takes more then 30 seconds; one-shot runs use a limit of 300 seconds. 
If a timeout occurs, the time and ground-file size collected up to the last completed step are still included in all averages.  

With Clingo, diminution cuts ground size by one to two orders of magnitude in the Automated Warehouse, HC, and SM domains, drops grounding time from more than a minute to a few seconds, and reduces the VH timeout rate from 62\% to under 5\%.  
DLV2 shows a similar pattern: on VH the timeout rate falls from 97\% to 15\%, and ground size shrinks five-fold.  
Across all benchmarks, using diminution ground programs translate into lower grounding and solving times as well as a significantly reduced ground file size, confirming that restricting the constant set effectively curbs both time and space overhead.

\section{Conclusion}
Diminution restricts the constant set before grounding, reducing ground programs while ensuring that every answer set still extends to an answer set of the original program. The transformation works entirely at the grounding stage and integrates with existing ASP solvers simply by adding domain predicates in the prescribed way. Looking ahead, we plan to add a neural network module that proposes promising constant subsets and iteratively refines them through solver feedback.

\section*{Acknowledgements}

We would like to express our sincere gratitude to Fangzhen Lin, Jiahuai You, and Yisong Wang for their constructive comments in the early stage of this work. We also appreciate Chenglin Wang for his helpful assistance in advancing the progress of this paper.

\clearpage
\appendix
\section*{Appendix}
\section{Algorithms mentioned in the paper}
This section presents concise pseudocode for algorithms referenced in the main paper: the Domain Predicate Lifting routine (Algorithm~\ref{alg:domlift}) and the bottom-up grounding procedure (Algorithm \ref{alg:dep-ground}).

\subsection{Domain Predicate Lifting Algorithm}

\begin{algorithm}[H]
\caption{\textsc{DomLift}$(P)$}
\label{alg:domlift}  
\begin{algorithmic}[1]
\REQUIRE  program $P$
\ENSURE   rewritten program $P^{\uparrow}$ with domain predicates

\STATE $P^{\uparrow} \gets P$
\FORALL{constant $c \in C(P)$}
  \STATE choose fresh variable $v_c$ and\ predicate $p_c$
  \STATE replace every occurrence of $c$ in $P^{\uparrow}$ by $v_c$
  \STATE add $p_c(c)$ as fact to $P^{\uparrow}$
  \STATE add $p_c(v_c)$ to each rule' body where $c$ was replaced
\ENDFOR
\RETURN $P^{\uparrow}$
\end{algorithmic}
\end{algorithm}
Algorithm \ref{alg:dep-ground} is mentioned in Theorem 4 of the main text. 
In brief, this procedure transforms $P$ into $P^\uparrow$, where for each constant $c\in C(P)$, it introduces a fresh variable $v_c$ and a domain predicate $p_c$, rewriting each occurrence of $c$ by $v_c$ guarded with $p_c(v_c)$, and adding the fact $p_c(c)$.  As a result, for any $\mathcal{D}\subseteq HU(P)$, each answer set of $P^\uparrow|_{\mathcal{D}}$ is obtained by extending some $I\in AS(P|_{\mathcal{D}})$ with the facts $\{p_c(c)\mid c\in C(P)\}$. Because the transformation does not introduce new constants, $HU(P^\uparrow)=HU(P)$.

\subsection{Basic Grounding Algorithm}
we present the bottom-up grounding algorithm that underlies modern ASP grounders and supports Definition 8 and Theorem 8 of the main text. The procedure given in Algorithm \ref{alg:dep-ground} follows the method of~\cite{gebser2022answer}, with only minor adjustments to fit our notation and presentation.

We first need to recast the grounding process from the perspective of substitution. For a function-free safe program $P$, A (ground) substitution is a mapping from variables to constant. Given two sets B and D of atoms, a \emph{substitution} $\theta$ is a \emph{match} of $B$ in $D$, if $B\theta \subseteq D$. A \emph{good match} is a $\subseteq-minimal$ one. Given a set $B$ of atoms(with variables) and a set $D$ of ground atoms, we define the set $\Theta(B, D)$ of good matches for all elements of $B$ in $D$.

\begin{algorithm}[H]
\caption{\textsc{GROUNDING}$(P)$}
\label{alg:dep-ground}
\begin{algorithmic}[1]
\REQUIRE Program $P$ with variables
\ENSURE  Ground program $P'$ 
\STATE Construct the predicate–rule dependency graph $G_{pr}$ from $P$
\STATE Let the SCCs of $G_{pr}$ be $C_1 \prec C_2 \prec \cdots \prec C_n$ in topological order
\STATE $P'\leftarrow\emptyset$; \quad $A_{\top}\leftarrow\emptyset$; \quad $A_{\neg \bot}\leftarrow\emptyset$

\FOR{$i=1$ to $n$}
  \IF{every element of $C_i$ is a ground fact $f$}
    \STATE $P'\leftarrow P'\cup\{f\}$;\\ $A_{\top}\leftarrow A_{\top}\cup\{f\}$;\\ $A_{\neg \bot}\leftarrow A_{\neg \bot}\cup\{f\}$
  \ELSIF{$C_i$ contains only predicate symbols}
    \STATE \textbf{continue}
  \ELSE
    \STATE $B\leftarrow \{body^+(r) \mid r \in C_i\}$; \quad \\
           $\Theta \leftarrow \Theta(B, A_{\neg \bot})$ 
       \FOR{$\theta \in \Theta$, $r\in C_i$}
       \STATE $r' \leftarrow r\theta$
       \IF{$body^+(r') \notin A_{\neg\bot}$ or \\ $body^-(r') \cap A_{\top} \not= \emptyset$ or \\ $head(r') \in A_{\top}$}
        \STATE \textbf{Continue}
       \ELSE
       \STATE $body^+(r')\leftarrow body^+(r') \setminus A_{\top}$ 
       ;\quad \\ $body^-(r')\leftarrow body^-(r') \cap A_{\neg\bot}$
       ;\quad \\ $P'\leftarrow P'\cup \{r'\}$ ;\quad $A_{\neg\bot} \leftarrow A_{\neg\bot} \cup head(r')$
           \IF{$body(r') = \emptyset$}
           \STATE $A_{\top} \leftarrow A_{\top}\cup head(r')$
           \ENDIF
       \ENDIF
       \ENDFOR
  \ENDIF
\ENDFOR
\RETURN $P'$
\end{algorithmic}
\end{algorithm}

The grounding procedure operates as follows.
\begin{enumerate}
\item Lines\,1–2 compute a topological order of the SCCs of the rule–predicate graph of given program; fixing the order in which components are grounded;
\item Line\,3 initializes two sets that are updated throughout the loop: $A_{\top}$ stores atoms are already known to be true, $A_{\neg\bot}$ stores atoms that possible true (unknown).
\item Because the input program is safe, grounding considers substitutions only for the variables that occur in positive body literals.  These variable-containing literals are instantiated only by substitutions $\theta \in \Theta(B,A_{\neg\bot})$, that is, substitutions that ground each positive body literal to an atom in $A_{\neg\bot}$.
\item During the loop (lines\,14–19) a candidate rule is discarded if its body cannot be satisfied or its head is already in $A_{\top}$.  For every rule that survives this check, body literals that are already satisfied are removed before the rule is added to $P'$.
\end{enumerate}

Furthermore, we slightly modify \textsc{Grounding}$(P)$ in
Algorithm~\ref{alg:dep-ground} by replacing line~10,
$\Theta \leftarrow \Theta(B, A_{\neg\bot})$
with
$\Theta \leftarrow \Theta(B, A_{\neg\bot}, \mathcal{D})$,
where
$  \Theta(B, A_{\neg\bot}, \mathcal{D}) =
  \{\,X \mapsto c \mid c \in \mathcal{D},  X \mapsto c \text{ is a } \text{good match of } B
      \text{ in } A_{\neg\bot}\}.$
We refer to the resulting procedure as
\textsc{Restrict-Grounding}\((P,\mathcal{D})\).
Intuitively, this change restricts variable instantiation in $P$ to
constants drawn exclusively from $\mathcal{D}$.

Consequently, for any program $P$, we have
$AS(\textsc{Restrict-Grounding}(P,\mathcal{D})) = AS(P|_{\mathcal{D}}),. $

\section{Detailed Proofs}
Here we provide full proofs omitted from the main paper for conciseness.
\subsection{Proof of Theorem 8}
\setcounter{theorem}{7} 
\begin{theorem}
    Given a program $P$ and its $\mathcal{D}$-guarded program $P^\mathcal{D}$, let
    \[ \mathcal{F}_{dom}=\{f\in P \mid pred(f) = \texttt{dom/1}\}.
    \]
    Then $AS(\texttt{Grounding}(P^\mathcal{[D]})\setminus \mathcal{F}_{dom})=AS(P|_\mathcal{D})$.
\end{theorem}
\begin{proof}
We write $P^{[\mathcal D]} = P_{dom} \cup P_{res}$, where $P_{dom}$ is a sub-program of $P^{[\mathcal D]}$ such that, for every rule $r_{dom}\in P_{dom}$, $\operatorname{pred}(r_{dom}) = {\texttt{dom/1}}$.

By definition of a $\mathcal{D}$-guarded program,fix an arbitrary topological order of the SCCs in the predicate–rule dependency graph of
$P^{[\mathcal{D}]}$.  For each $C_i$ containing $\texttt{dom}/1 \in pred(C_i)$ we have:
\begin{itemize}
\item no $C_j \prec C_i$ contains a rule $r$ with
      $V(r)\neq\emptyset$ whose head predicate differs from $\texttt{dom}/1$; and
\item no $C_i$ simultaneously contains $\texttt{dom}/1$ and a predicate
      $p \in \operatorname{pred}(P)$.
\item in each immediate successor of $C_i$, every variable $X$ occurring in a rule $r$ is guarded by the literal $\texttt{dom}(X)$ in $body^{+}(r)$.
\end{itemize}
By definition, $\mathcal{D}$ is the unique instantiation domain for the predicate \texttt{dom}/1. 
Further, before the algorithm reaches any component $C_i$, \textsc{Grounding} has processed only ground rules or rules $r$ with $head(r) = \texttt{dom}/1$.)
Consequently, after $C_i$ is processed, the sets $A_{\top}$ and $A_{\neg\bot}$ contain the same ground atoms of the form \texttt{dom}$(\cdot)$.

Thus, when \textsc{Grounding}$(P)$ arrives at a component $C_k$ containing a rule with variables and $pred(head(r)) \neq \texttt{dom}/1$, every ground atom \texttt{dom}$(c)$ is already in $A_{\top}$ and will not be enlarged further.
For any component $C_k$, one can choose an order $C_i\prec\cdots\prec C_k$ so that grounding under this order never introduces a constant outside $\mathcal D$ into $A_{\neg\bot}$.
Hence $\Theta \gets \Theta(B,A_{\neg\bot})$ has the same effect as $\Theta \gets \Theta(B,A_{\neg\bot},\mathcal D)$.

Afterwards, grounding $P^{[\mathcal D]}$ turns every rule in $P_{dom}$ into a fact and removes every literal \texttt{dom}$(X)$ from the body of each rule in $P_{\mathrm{res}}$.
Consequently,
$\textsc{Grounding}(P^{[\mathcal D]}) \setminus \mathcal F_{dom} =
\textsc{Restrict-Grounding}(P,\mathcal D)$,
so the two programs possess identical answer sets.)
\end{proof}

\section{Additional Experimental Results}
We report more fine‐grained results and analyses.

\subsection{Experimental Setups} 

Table \ref{tab:incmode-scenarios} lists every benchmark used under the
\textsc{incmode} setting.  We cover five families:

\begin{table}[H]
  \centering
  \small
  \setlength{\tabcolsep}{1.3mm}
  \begin{tabular}{@{}l l r@{}}
    \toprule
    \textbf{Env.} & \textbf{Configuration Values} & \textbf{\#Inst.} \\
    \midrule
    VH & --- & 643 \\ \midrule
    \multirow{4}{*}{AWS}
      & (8, 8, 3, 5, 3, 3) & 50 \\
      & (10, 10, 5, 10, 3, 3)  & 50 \\
      & (12, 12, 6, 30, 6, 6) & 50 \\
      & (15, 15, 7, 30, 5, 5) & 50 \\ \midrule
    \multirow{2}{*}{GW}
      & (50, 50, 2, 2, 8)    & 50 \\
      & (100, 100, 8, 2, 8)  & 50 \\ \midrule
    \multirow{4}{*}{HC}
      & 200 & 10 \\
      & 400 & 10 \\
      & 600 & 10 \\
      & 800 & 10 \\ \midrule
    \multirow{4}{*}{SM}
      & 30  & 10 \\
      & 60  & 10 \\
      & 90  & 10 \\
      & 120 & 10 \\
    \bottomrule
  \end{tabular}
  \caption{Benchmarks used in the \textsc{incmode} evaluation.  
           \textbf{AWS}: \textit{(width, height, product types, shelves, orders, max products/order)};  
           \textbf{GW}:  \textit{(width, height, obstacles, min size, max size)};  
           \textbf{HC}: single value = number of nodes;  
           \textbf{SM}: single value = number of people (men = women).}
  \label{tab:incmode-scenarios}
\end{table}

\subsection{Fine-Grained Analysis of Behavior in \texttt{incmode}}

Throughout the figures in this section, we label each configuration with two tags:
\texttt{clingo} or \texttt{DLV2} (the ASP solver) and
\texttt{Diminution} or $HU(P)$ (the grounding strategy).
Here \texttt{Diminution} means grounding is restricted to a selected constant subset~$\mathcal D\subseteq HU(P)$.
For example, \texttt{clingo}/\texttt{Diminution} denotes
running \texttt{clingo} on a $\mathcal D$-guarded (diminished) version of~$P$.

\begin{figure}[H]
  \centering

  \begin{subfigure}[b]{\linewidth}
    \centering
    \includegraphics[height=5.5cm]{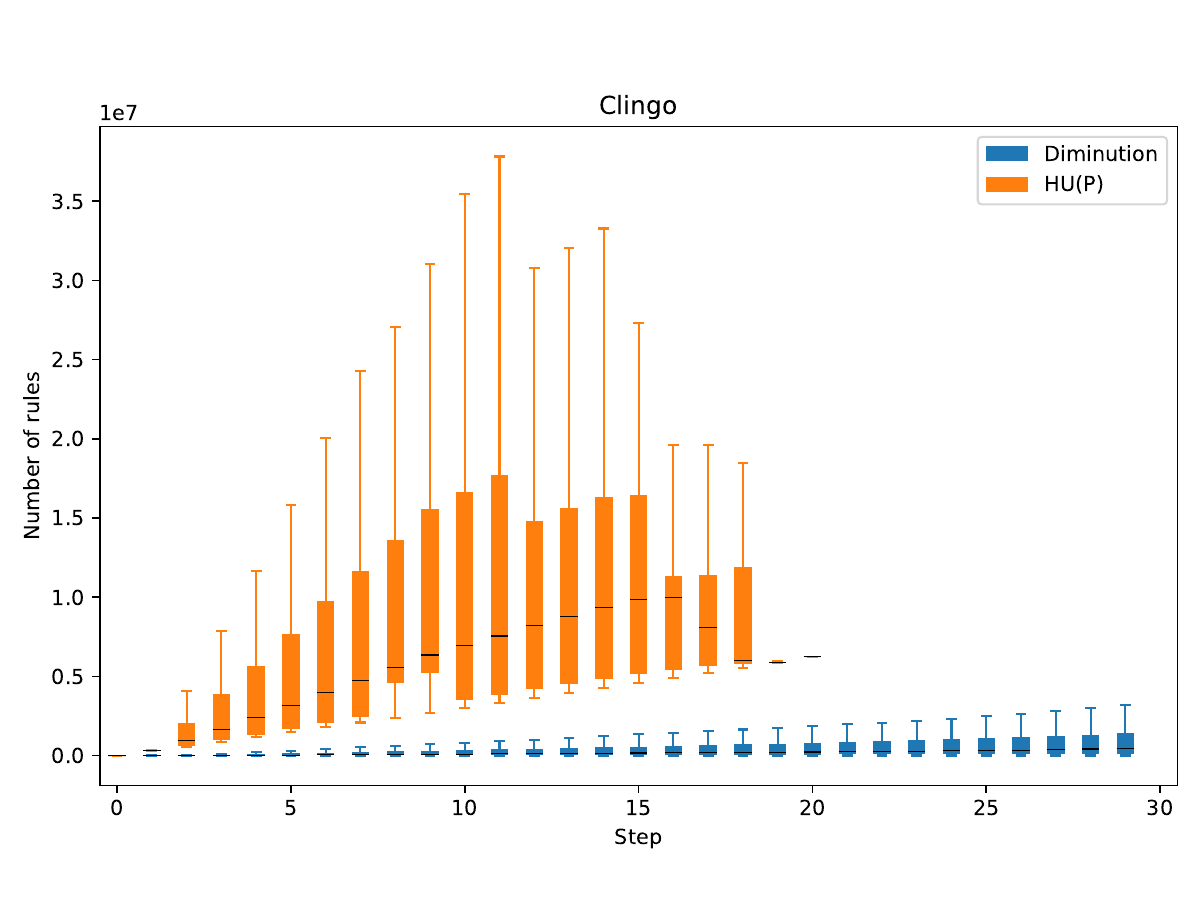}
    \caption{\textbf{VH}}
  \end{subfigure}

  \begin{subfigure}[b]{\linewidth}
    \centering
    \includegraphics[height=5.5cm]{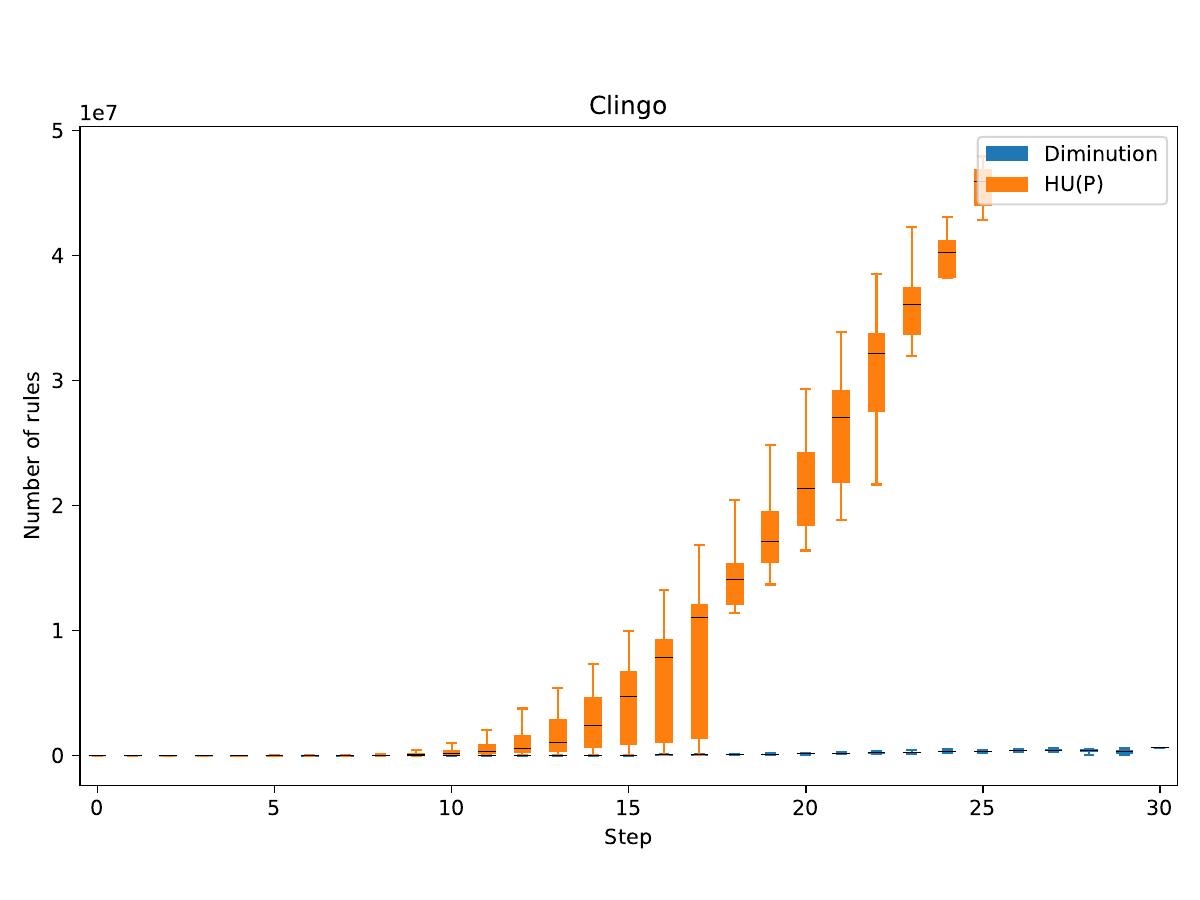}
    \caption{\textbf{AWS}}
  \end{subfigure}

  \begin{subfigure}[b]{\linewidth}
    \centering
    \includegraphics[height=5.5cm]{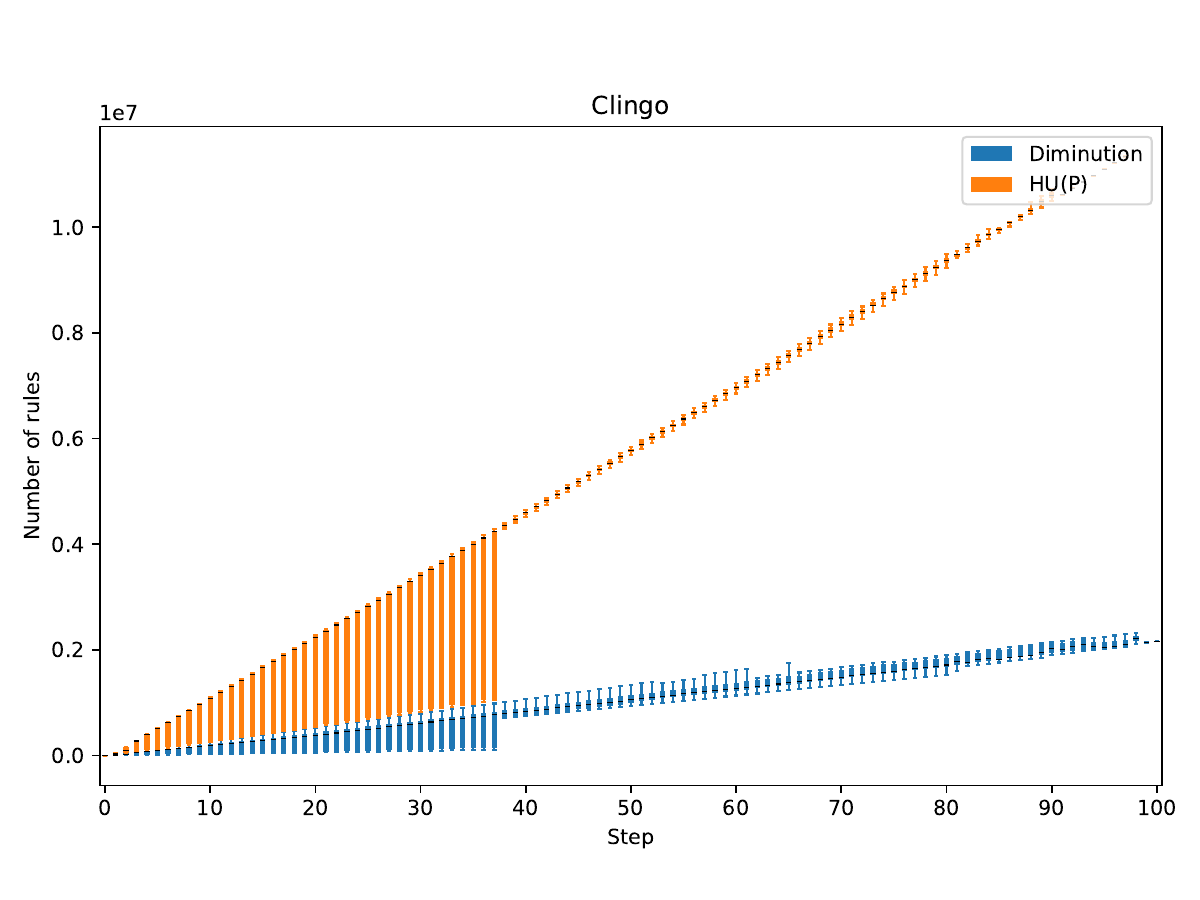}
    \caption{\textbf{GW}}
  \end{subfigure}

  \caption{Step-wise ground-rule counts produced by \texttt{incmode} solving for the three planning domains.}
  \label{fig:incmode_num_rules}
\end{figure}

\begin{figure}[H]
  \centering

  \begin{subfigure}[b]{\linewidth}
    \centering
    \includegraphics[height=5.5cm]{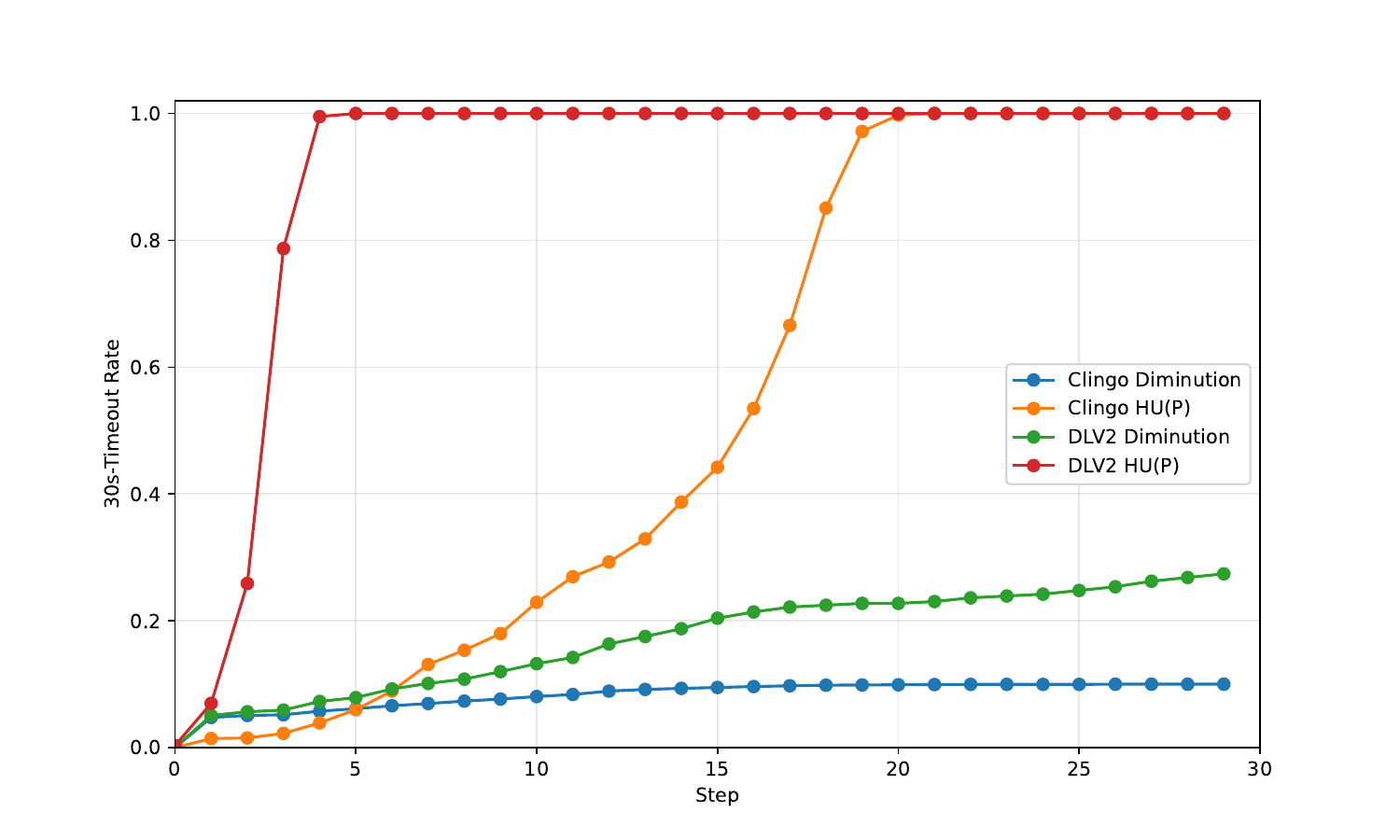}
    \caption{\textbf{VH}}
    \label{fig:vh_timeout}
  \end{subfigure}

  \begin{subfigure}[b]{\linewidth}
    \centering
    \includegraphics[height=5.5cm]{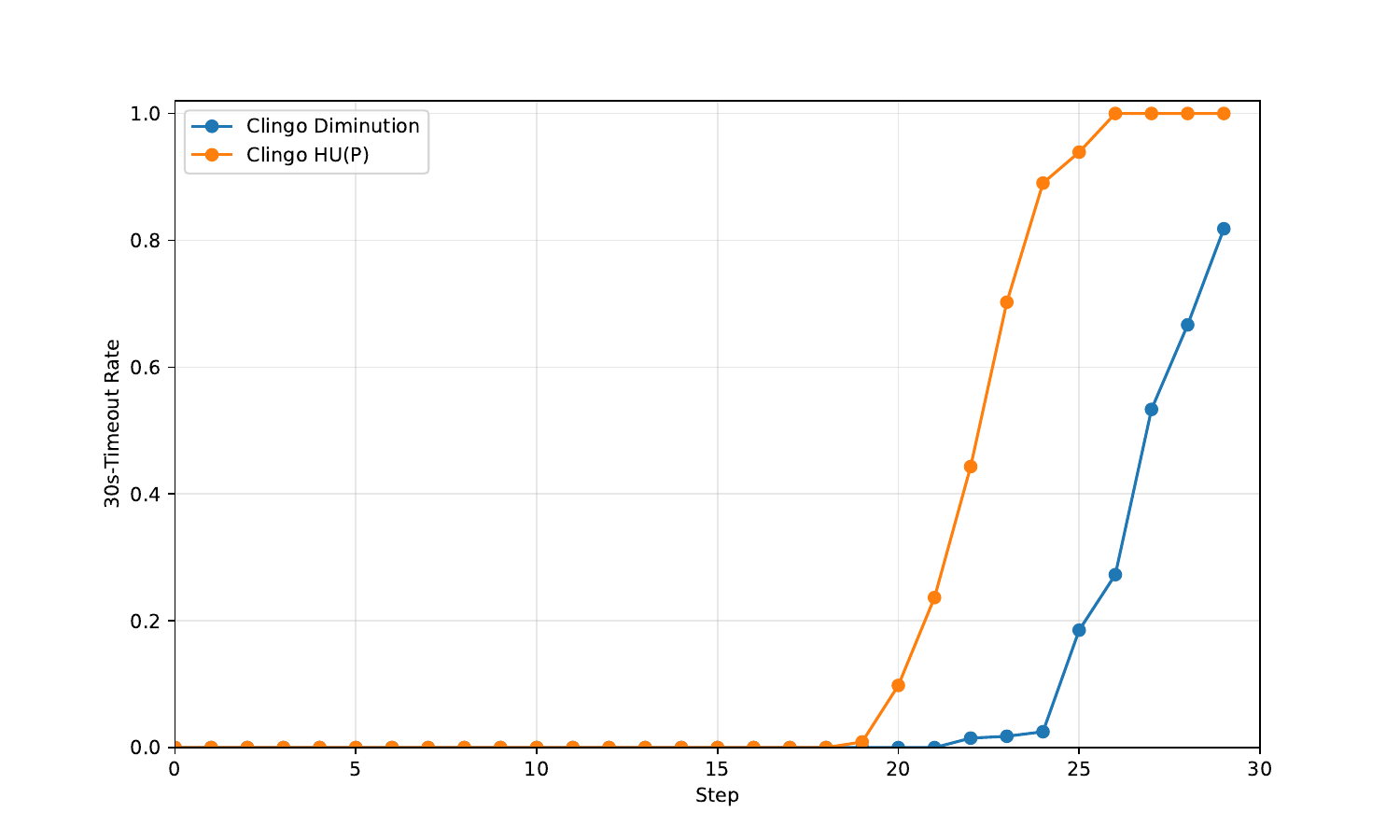}
    \caption{\textbf{AWS}}
    \label{fig:aws_timeout}
  \end{subfigure}

  \begin{subfigure}[b]{\linewidth}
    \centering
    \includegraphics[height=5.5cm]{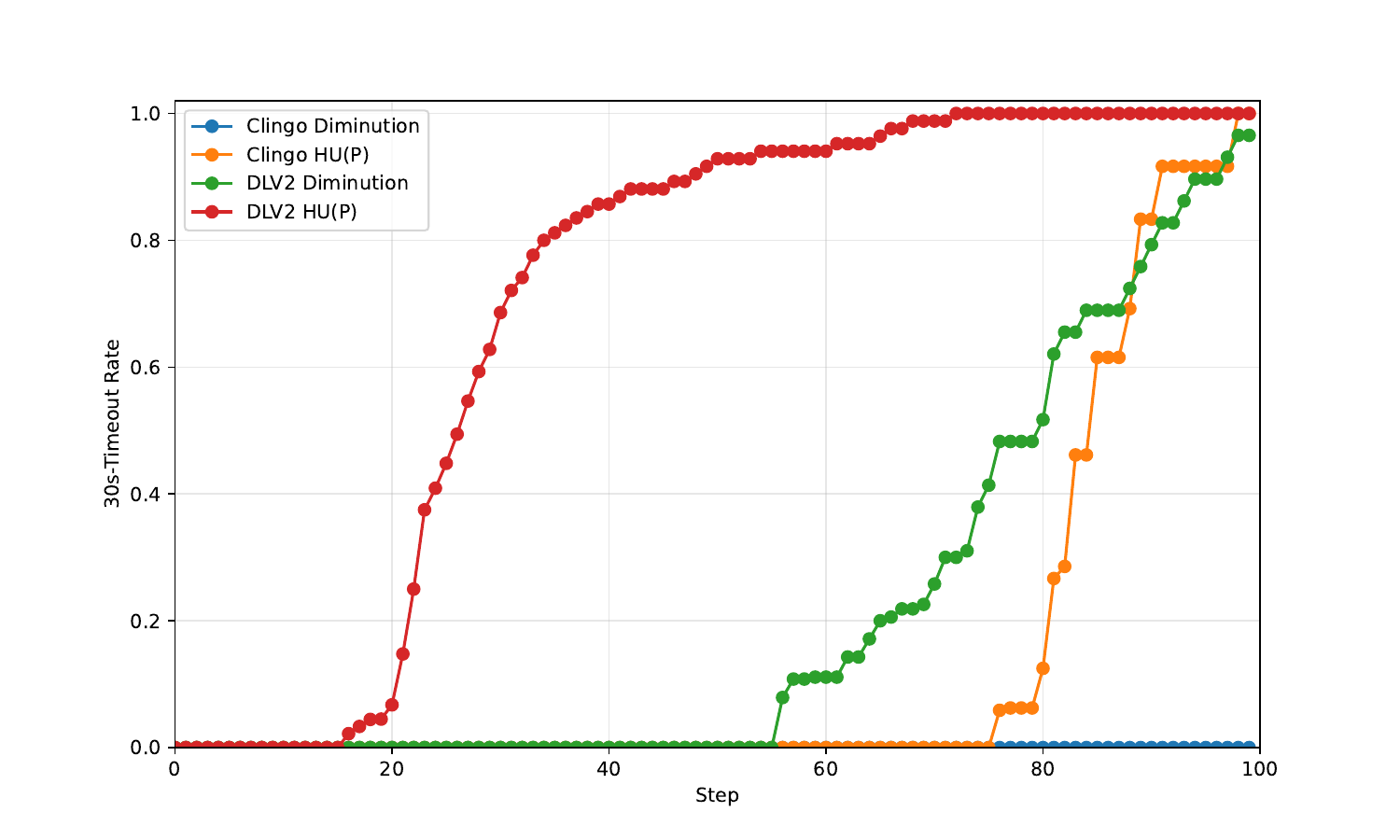}
    \caption{\textbf{GW}}
    \label{fig:gw_timeout}
  \end{subfigure}

  \caption{Step-wise $30\,\text{s}$ timeout rates of \texttt{incmode} solving for three planning domains.}
  \label{fig:incmode_timeout_rates_single}
\end{figure}

\clearpage

\begin{figure*}
  \centering

  \begin{subfigure}[b]{\linewidth}
    \centering
    \includegraphics[height=5.5cm]{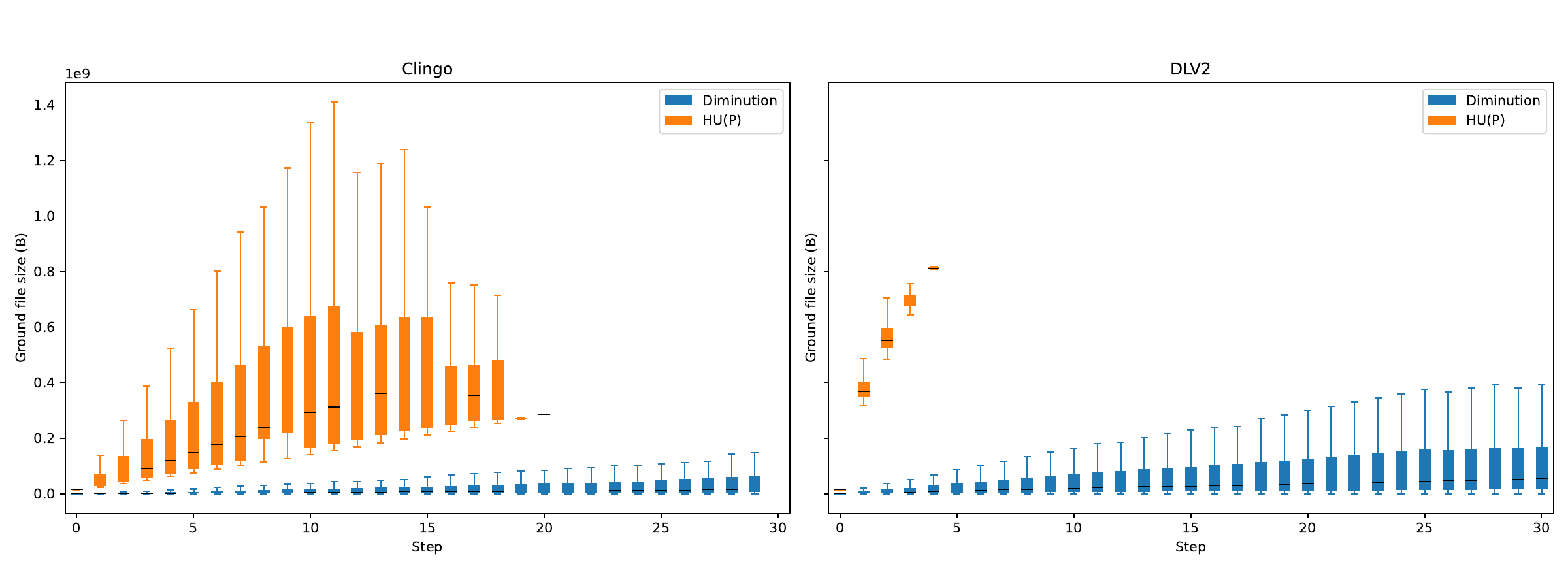}
    \caption{\textbf{VH}}
  \end{subfigure}

  \begin{subfigure}[b]{\linewidth}
    \centering
    \includegraphics[height=5.5cm]{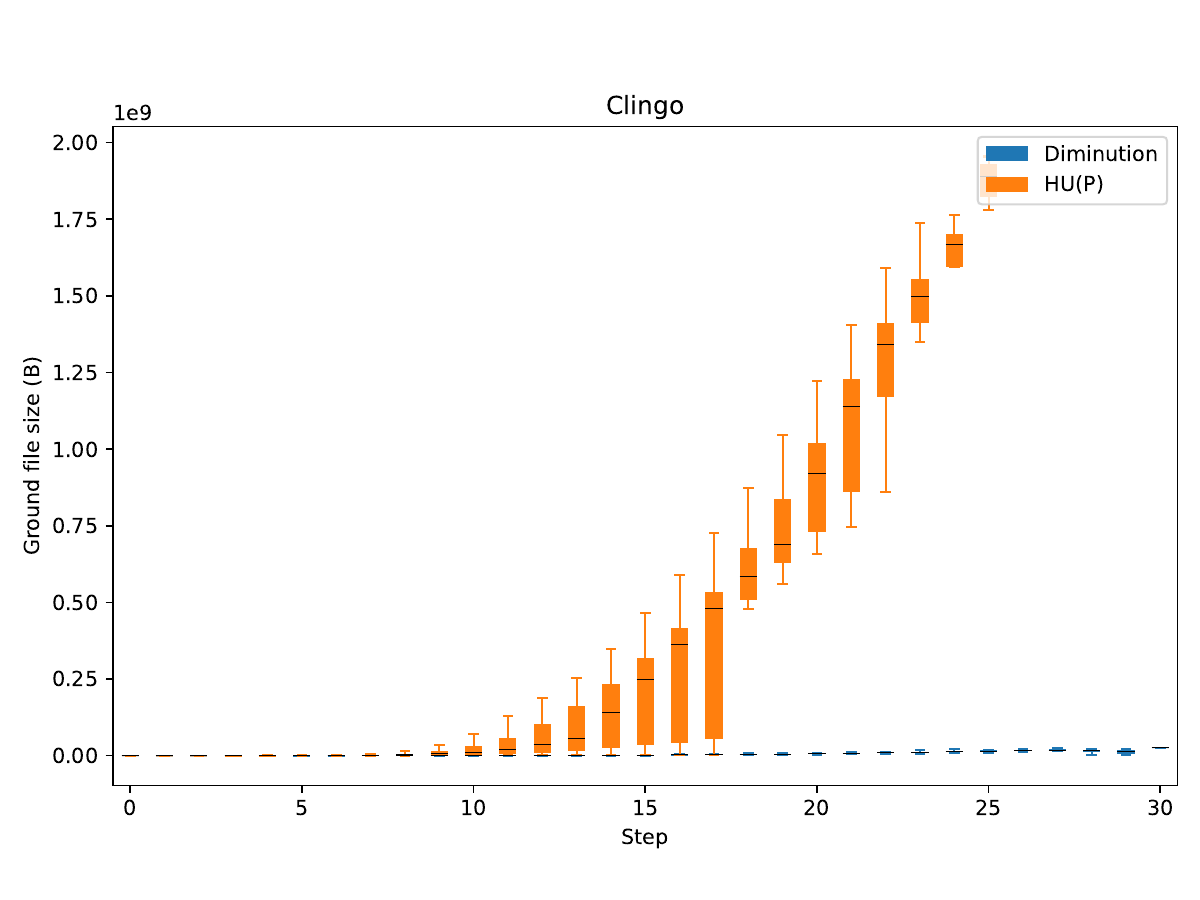}
    \caption{\textbf{AWS}}
  \end{subfigure}

  \begin{subfigure}[b]{\linewidth}
    \centering
    \includegraphics[height=5.5cm]{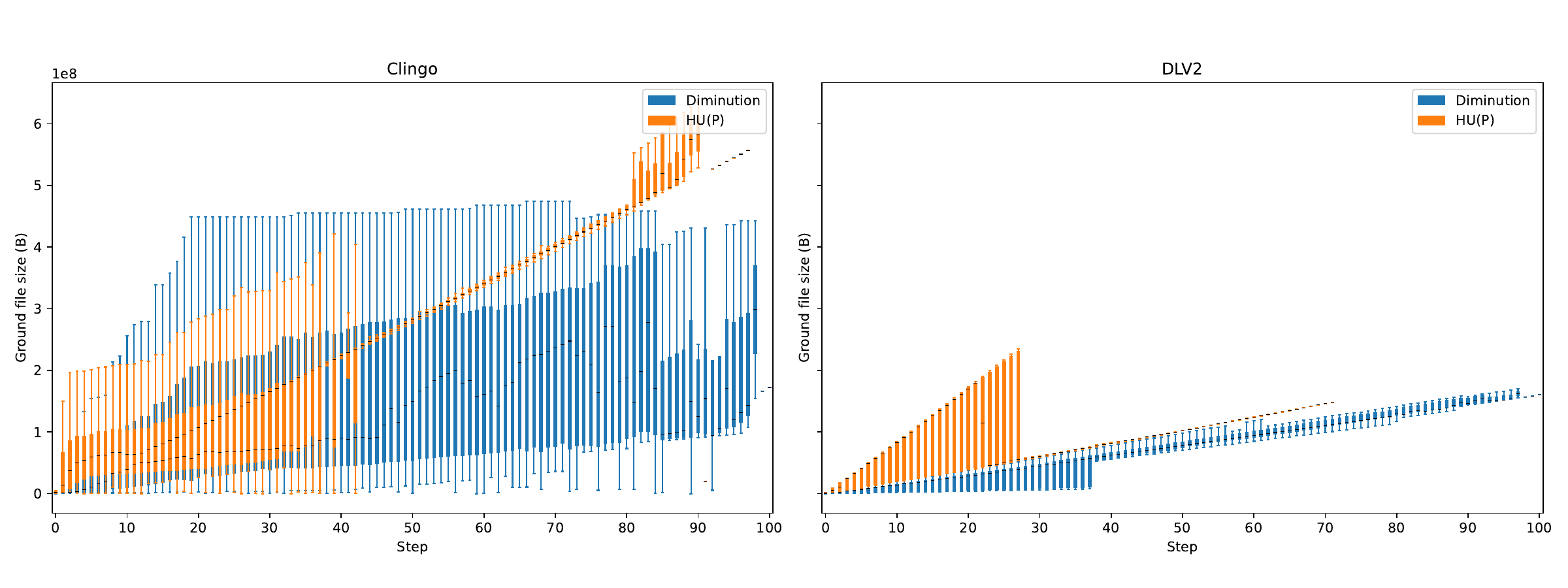}
    \caption{\textbf{GW}}
  \end{subfigure}

  \caption{Step-wise \texttt{aspif} ground-file size produced by \texttt{incmode} solving for the three planning domains.}
  \label{fig:incmode_file_size_rules}
\end{figure*}

\begin{figure*}
  \centering

  \begin{subfigure}[b]{\linewidth}
    \centering
    \includegraphics[height=5.5cm]{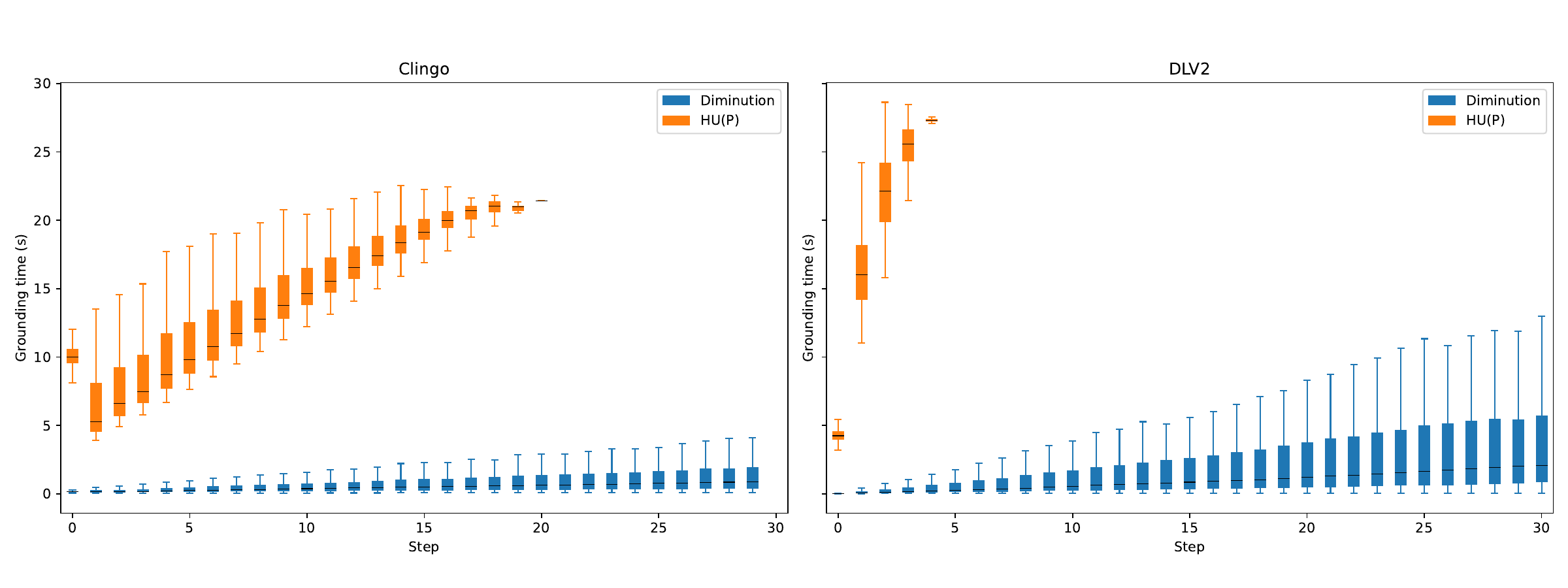}
    \caption{\textbf{VH}}
  \end{subfigure}

  \begin{subfigure}[b]{\linewidth}
    \centering
    \includegraphics[height=5.5cm]{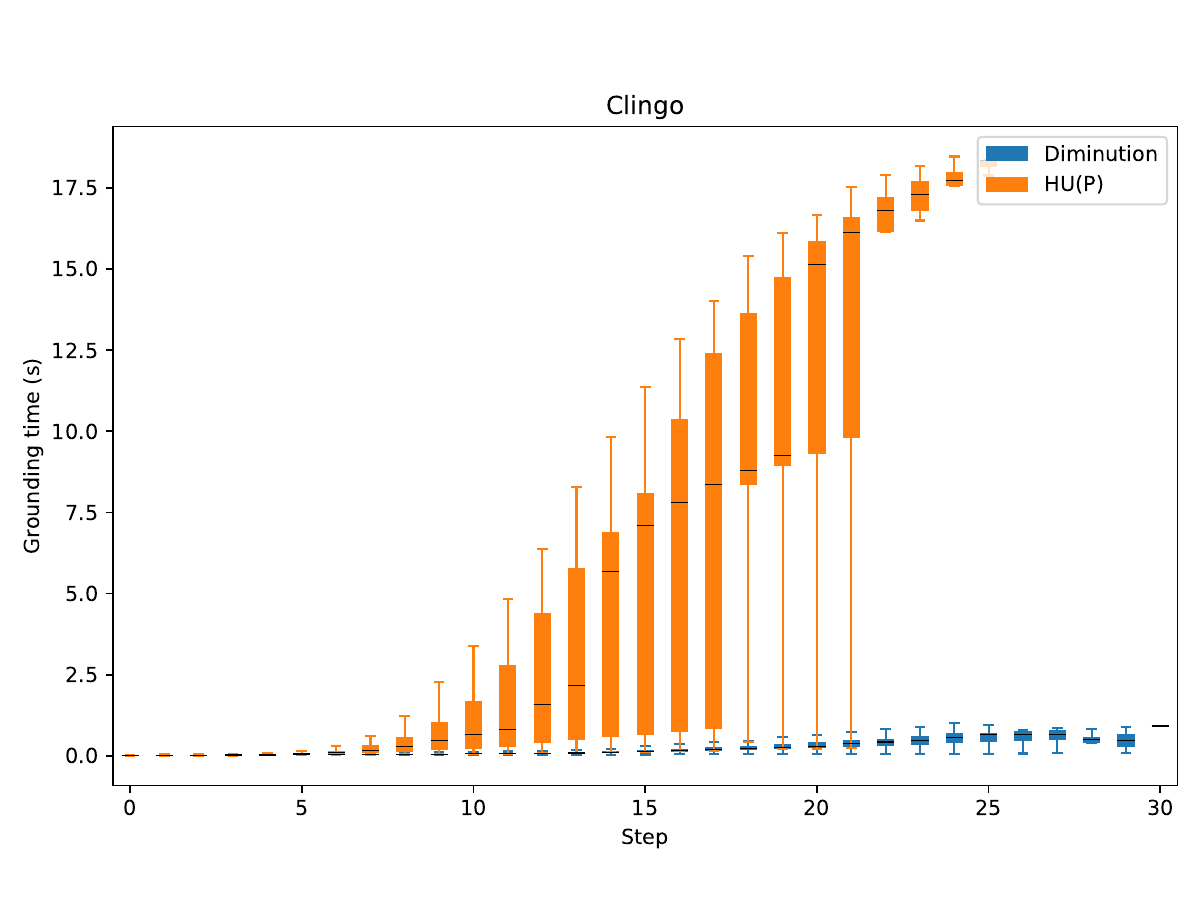}
    \caption{\textbf{AWS}}
  \end{subfigure}

  \begin{subfigure}[b]{\linewidth}
    \centering
    \includegraphics[height=5.5cm]{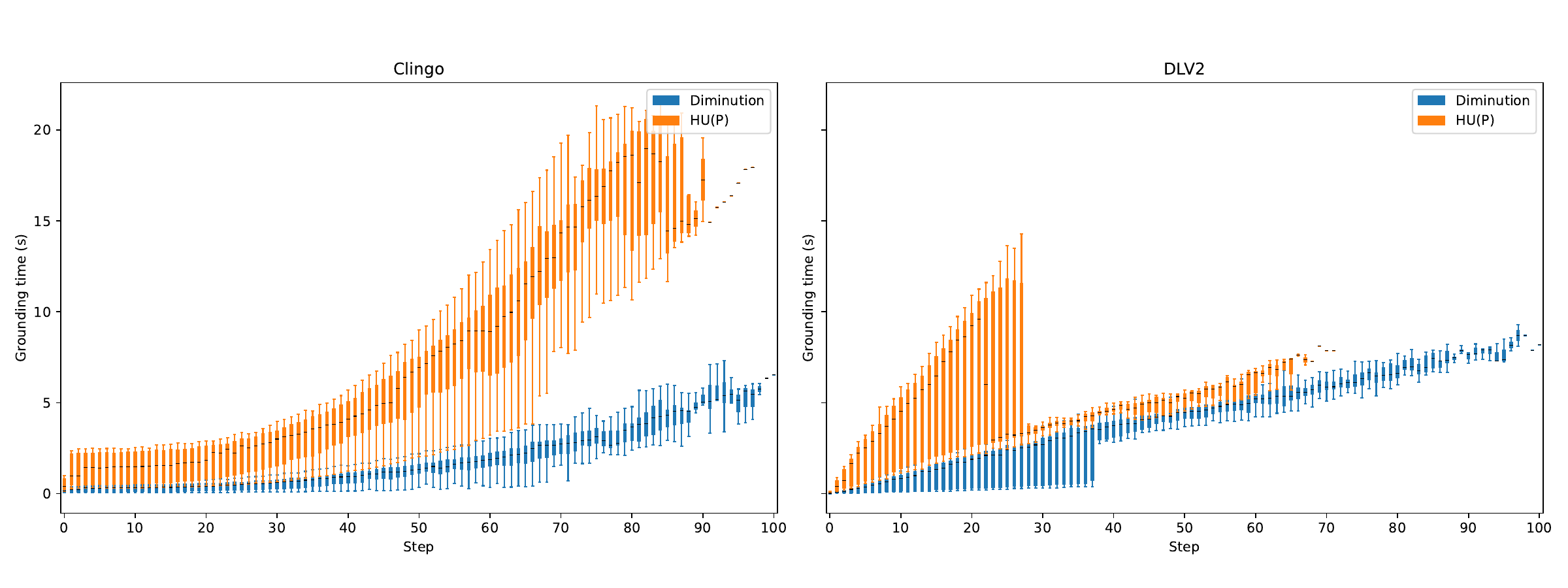}
    \caption{\textbf{GW}}
  \end{subfigure}

  \caption{Step-wise \texttt{aspif} ground-file size produced by \texttt{incmode} solving for the three planning domains.}
  \label{fig:incmode_ground_time}
\end{figure*}

\begin{figure*}
  \centering

  \begin{subfigure}[b]{\linewidth}
    \centering
    \includegraphics[height=5.5cm]{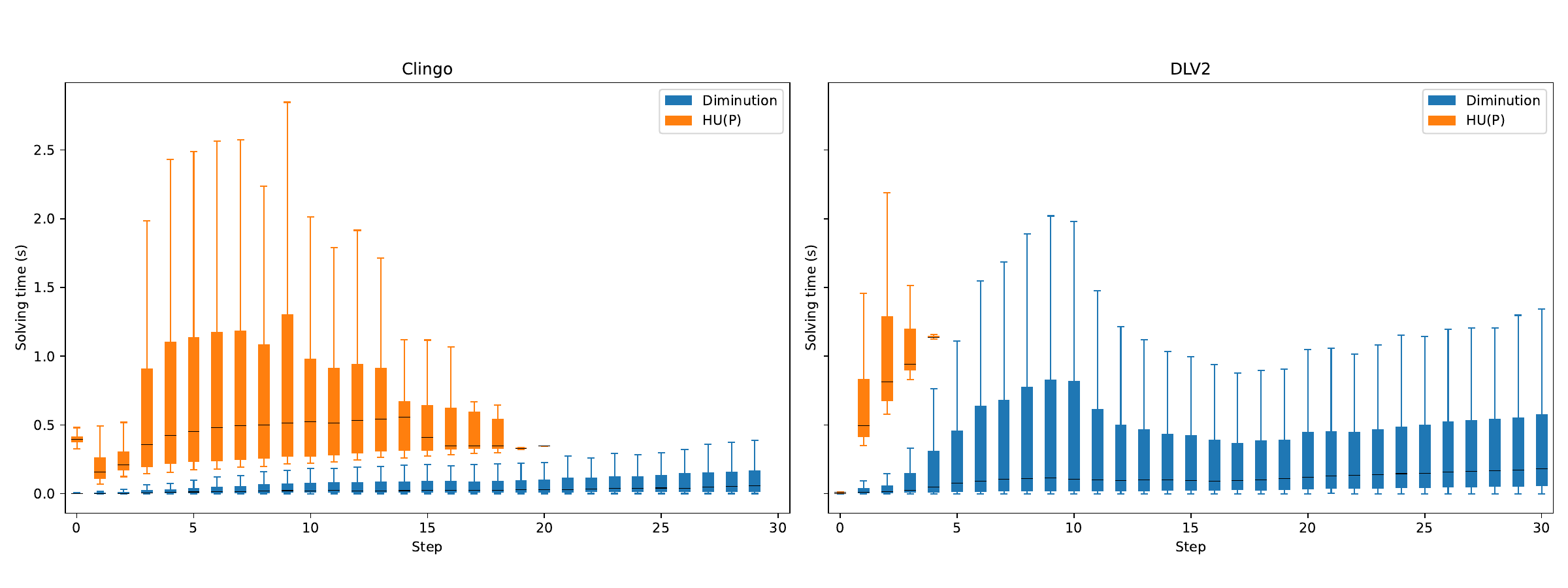}
    \caption{\textbf{VH}}
  \end{subfigure}

  \begin{subfigure}[b]{\linewidth}
    \centering
    \includegraphics[height=5.5cm]{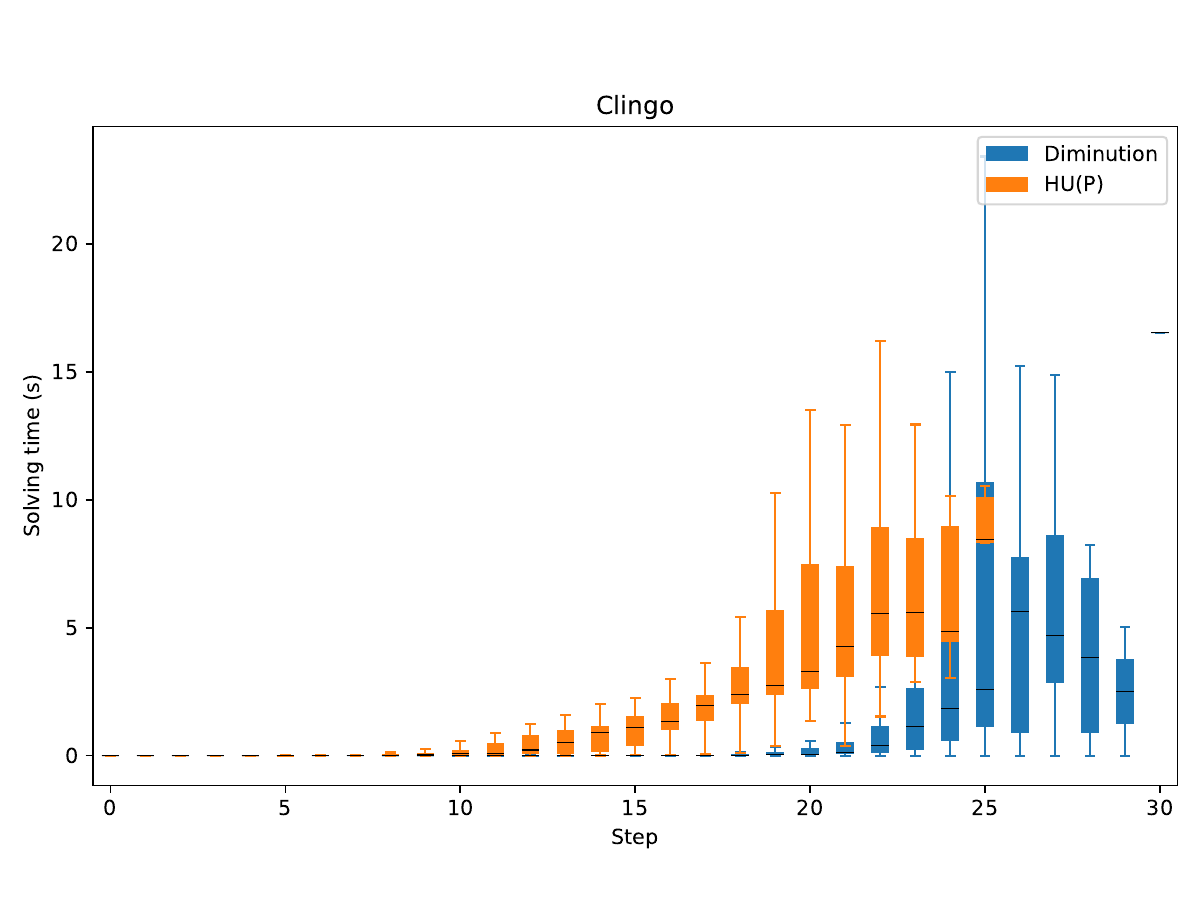}
    \caption{\textbf{AWS}}
  \end{subfigure}

  \begin{subfigure}[b]{\linewidth}
    \centering
    \includegraphics[height=5.5cm]{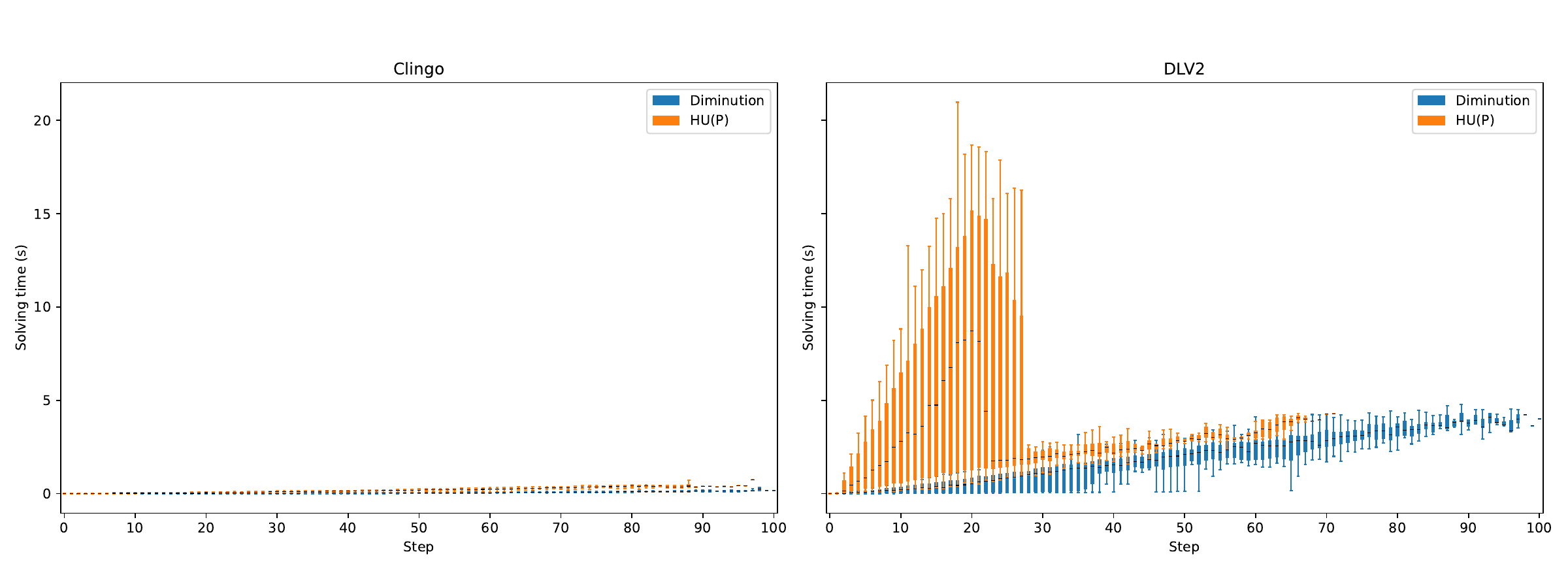}
    \caption{\textbf{GW}}
  \end{subfigure}

  \caption{Step-wise \texttt{aspif} ground-file size produced by \texttt{incmode} solving for the three planning domains.}
  \label{fig:incmode_solve_time}
\end{figure*}

We begin by tracking how problem size changes with step number in each \texttt{incmode} domain (\textbf{VH}, \textbf{AWS}, \textbf{GW}). Size is measured by the number of ground rules at each step shown in figure~\ref{fig:incmode_num_rules} and the byte size of the resulting \texttt{aspif} file figure~\ref{fig:incmode_file_size_rules}. Both metrics are shown as step-wise boxplots, each box capturing the distribution of instances still running at that step.

Figure \ref{fig:incmode_timeout_rates_single} plots the step-wise timeout rate. A run is counted as a timeout when its total wall-time passes 30 s. Small overheads—such as \texttt{DLV} parsing—can nudge the wall-time just over 30 s even if solving completes a bit sooner, but the difference is negligible.

Figures~\ref{fig:incmode_ground_time} and~\ref{fig:incmode_solve_time} present step-wise boxplots of \emph{grounding time} and \emph{solving time},
respectively.

Across all three domains, grounding with \emph{diminution} systematically yields
smaller ground programs, lowers both grounding and solving times, and thus
pushes the timeout rate well below that of the full-universe baseline $HU(P)$.
These advantages persist step-by-step—even in \textbf{GW}, whose longer
episodes ($0$–$100$ steps) amplify absolute costs—because every diminished
instance is evaluated on the same horizon as its baseline counterpart.
A few irregularities do appear: the \emph{rule-count} metric is available only
for \texttt{clingo}, and the \textbf{AWS} domain is likewise limited to
\texttt{clingo}, so those plots omit \texttt{DLV2}; moreover, at later steps
the sample size contracts as more runs finish or time out, occasionally
inflating variance or causing a slight dip in scale metrics, especially where
the remaining solvable instances are inherently simpler.
Taken together, however, the results leave little doubt that diminution is the
more efficient grounding strategy, delivering leaner encodings and faster
overall solving without compromising completeness.

\end{document}